\documentclass[letterpaper]{article} % DO NOT CHANGE THIS
\usepackage[]{aaai23}  % DO NOT CHANGE THIS
\usepackage{times}  % DO NOT CHANGE THIS
\usepackage{helvet}  % DO NOT CHANGE THIS
\usepackage{courier}  % DO NOT CHANGE THIS
\usepackage[hyphens]{url}  % DO NOT CHANGE THIS
\usepackage{graphicx} % DO NOT CHANGE THIS
\usepackage{natbib}  % DO NOT CHANGE THIS AND DO NOT ADD ANY OPTIONS TO IT
\usepackage{caption} % DO NOT CHANGE THIS AND DO NOT ADD ANY OPTIONS TO IT
\usepackage{algorithm}
\usepackage{algorithmic}
\usepackage[utf8]{inputenc}
\usepackage{amsmath}
\usepackage{amsfonts}
\usepackage{multirow}
\usepackage{graphicx}
\usepackage{amsthm}
\usepackage{subfig}
\theoremstyle{definition}
\newtheorem{definition}{Definition}
\newtheorem{theorem}{Theorem}
\usepackage{xspace}
\usepackage{color}
\usepackage{booktabs}
\usepackage{multirow}
\usepackage{newfloat}
\usepackage{listings}
\DeclareCaptionStyle{ruled}{labelfont=normalfont,labelsep=colon,strut=off} % DO NOT CHANGE THIS
\floatstyle{ruled}
\newfloat{listing}{tb}{lst}{}
\floatname{listing}{Listing}

\newcommand{\datasetFont}{\text}
\newcommand{\ours}{\datasetFont{UR-ERN}\xspace}

\title{Uncertainty Regularized Evidential Regression}
\author{
    Kai Ye\textsuperscript{\rm 1}, Tiejin Chen\textsuperscript{\rm 2}, Hua Wei\textsuperscript{\rm 2$*$}, Liang Zhan\textsuperscript{\rm 1}\thanks{Corresponding author.}
}
\affiliations{
    \textsuperscript{\rm 1}University of Pittsburgh, Pittsburgh, PA, 15260, USA\\
    \textsuperscript{\rm 2}Arizona State University, Tempe, AZ, 85281, USA\\

    hua.wei@asu.edu,
    liang.zhan@pitt.edu 
    
}

\usepackage{bibentry}
\begin{document}

\maketitle

\begin{abstract}
% Evidential Regression Network (ERN) represents a novel approach that integrates deep learning with Dempster-Shafer theory to predict a target and quantify the associated uncertainty. Guided by the underlying theory, specific activation functions must be employed to enforce non-negative values, a constraint that compromises model performance. However, this constraint leads to models can't learn from all the samples, therefore the model performance is compromised. In this paper, we provide a theoretical analysis of this limitation and provide our improvement. We first give a definition of the area where the models can't effectively learn from samples. Then we theoretically analyze ERN and investigate this limitation. Based on the analysis, we tackle the limitation by proposing a regularization term that helps ERN learn from all the samples. Finally, we conduct extensive experiments to confirm our theoretical analysis and demonstrate the effectiveness of our proposed solution.
The Evidential Regression Network (ERN) represents a novel approach that integrates deep learning with Dempster-Shafer's theory to predict a target and quantify the associated uncertainty. Guided by the underlying theory, specific activation functions must be employed to enforce non-negative values, which is a constraint that compromises model performance by limiting its ability to learn from all samples. This paper provides a theoretical analysis of this limitation and introduces an improvement to overcome it. Initially, we define the region where the models can't effectively learn from the samples. Following this, we thoroughly analyze the ERN and investigate this constraint. Leveraging the insights from our analysis, we address the limitation by introducing a novel regularization term that empowers the ERN to learn from the whole training set. Our extensive experiments substantiate our theoretical findings and demonstrate the effectiveness of the proposed solution.
\end{abstract}

\section{Introduction}
Deep learning methods have been successful in a broad spectrum of real-world tasks, including computer vision~\cite{godard2017unsupervised,he2016deep}, natural language processing~\cite{zhao2023evaluating,devlin2018bert,vaswani2017attention}, and medical domain~\cite{ye2023bidirectional}. %In these contexts, the assessment of model uncertainty emerges as a vital component. Uncertainty within the context of deep learning can typically be classified into two principal categories: the irreducible inherent randomness of data, the \textit{aleatoric uncertainty}, and the uncertainty related to model parameters, the \textit{epistemic uncertainty}~\cite{gal2016dropout,guo2017calibration}.
In these scenarios, evaluating model uncertainty becomes a crucial element. Within the realm of deep learning, uncertainty is generally categorized into two primary groups: the intrinsic randomness inherent in data, referred to as the \textit{aleatoric uncertainty}, and the uncertainty associated with model parameters, known as the \textit{epistemic uncertainty}~\cite{gal2016dropout,guo2017calibration}.

%Among these, the representation of epistemic uncertainty presents a particular challenge, owing to the inherent difficulty associated with accurately quantifying the uncertainty tied to the model's parameters. 
Among these, accurately quantifying the uncertainty linked to the model's parameters proves to be a particularly demanding task, due to the inherent complexity involved.
To tackle this, strategies such as Ensemble-based methods~\cite{pearce2020uncertainty,lakshminarayanan2017simple} and Bayesian neural networks (BNNs)~\cite{gal2016dropout,wilson2020bayesian,blundell2015weight} have been proposed to measure epistemic uncertainty. 
%However, they either require extensive computational cost or have scalability issues. To solve these issues, evidential deep learning methods~\cite{sensoy2018evidential,NEURIPS2020_aab08546,malinin2018predictive} have been conceived, designed to address uncertainty estimation by generating parameters of distributions as its outputs.
Nonetheless, these methods either demand substantial computational resources or encounter challenges in scalability. In response to these limitations, the concept of evidential deep learning techniques~\cite{sensoy2018evidential,NEURIPS2020_aab08546,malinin2018predictive} has emerged. These methods are formulated to handle uncertainty estimation by producing distribution parameters as their output.

% instrumental in both the evaluation of model performance and the interpretation of prediction confidence. 
% Addressing this necessity, evidential deep learning methods have been conceived, specifically targeting the estimation of uncertainty.

\begin{figure}
  \centering
  \begin{tabular}{c}
       \includegraphics[width=0.9\linewidth]{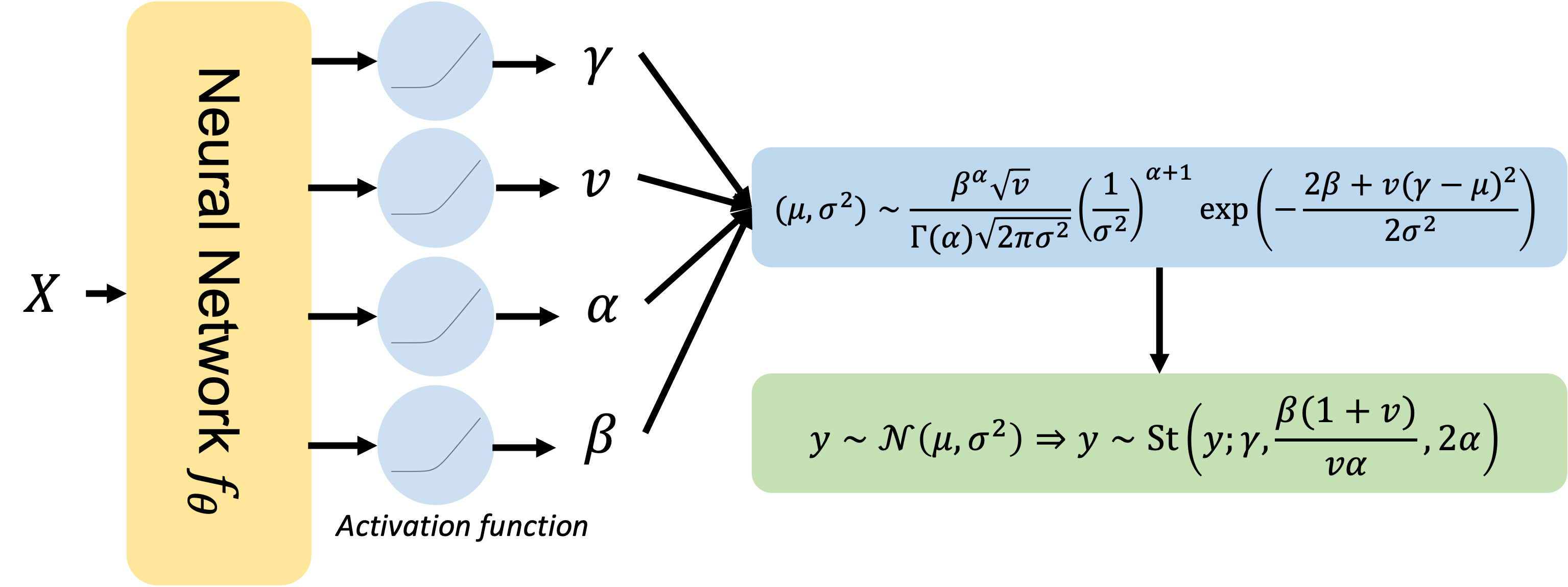}  \\
       (a) Evidential Regression Network (ERN) architecture \\
       \includegraphics[width=0.9\linewidth]{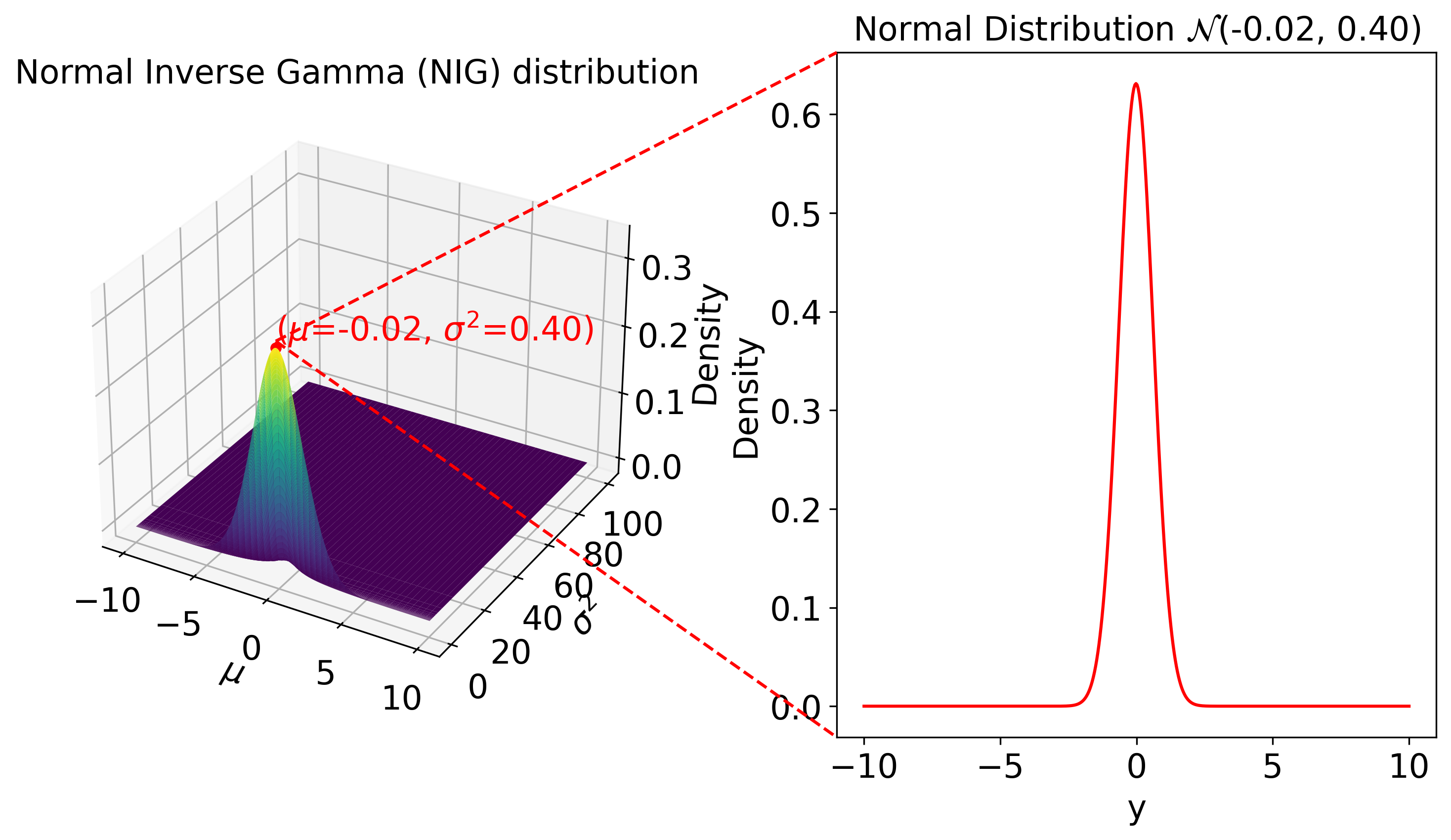}  \\
       (b) Normal Inverse-Gamma (NIG) distribution.\\
  \end{tabular}
  \caption{ An overview of the Evidential Regression Network (ERN) architecture with illustrations on the final distributions of the prediction. ERN outputs four predictions as distribution parameters, with activation functions like Relu or Softplus to constrain the output to meet the requirements of distribution parameters.}
  \label{uncertainty_framework}
\end{figure}

The Evidential Regression Network (ERN)~\cite{NEURIPS2020_aab08546} introduces a novel deep-learning regression approach that incorporates Dempster-Shafer theory~\cite{shafer1976mathematical} to quantify model uncertainty, resulting in impressive achievements. Within the ERN framework, the training process is conceptualized as an evidence acquisition process, which is inspired by evidential models for classification~\cite{malinin2018predictive,malinin2019reverse,bilovs2019uncertainty,haussmann2019bayesian,malinin2019ensemble}. 
During the training phase, ERN establishes prior distributions over the likelihood function, and each training sample contributes to the formation of a higher-order evidential distribution from which the likelihood function is drawn. 
During the inference phase, ERN produces the hyperparameters of the evidential distribution, facilitating both prediction and uncertainty estimation without the necessity for sampling. 
This approach was subsequently extended to multivariate regression tasks by~\citeauthor{meinert2021multivariate} using different prior distributions. 

Previous ERN methods \cite{NEURIPS2020_aab08546,malinin2020regression,charpentier2021natural,oh2022improving,feng2023deep,mei2023uncertainty} use specific activation functions like ReLU to ensure non-negative values for parameters of the evidential distribution, such as the variance.
Nevertheless, the utilization of such activation functions may inadvertently hinder ERN models' capacity to learn effectively from training samples, thereby impairing overall model performance~\cite{pandey2023learn}. 
%In classification, evidential models have shown suboptimal results because of `zero confidence regions' within the evidential space
%In classification, evidential models have demonstrated suboptimal results due to the presence of `zero confidence regions` within the evidential space~\cite{pandey2023learn}. 
Furthermore, in classification tasks, evidential models have underperformed because of the existence of "zero confidence regions" within the evidential space~\cite{pandey2023learn}.

%However, analysis for evidential models in the context of regression tasks is notably lacking.

However, there is a notable lack of convergence analysis for evidential models in the context of regression tasks.
In this paper, we explore the existence of zero confidence regions, which result in high uncertainty areas (HUA) during the training process of ERN models for regression tasks.
Building upon the insights derived from our analysis, we propose a novel regularization term that enables the ERN to bypass the HUA and effectively learn from the zero-confidence regions. 
We also show that the proposed regularization can be generalized to various ERN variants.
We conduct experiments on both synthetic and real-world data and show the effectiveness of the proposed method\footnote{Code is at https://github.com/FlynnYe/UR-ERN}. 
The main contributions of our work are summarized as follows:
\begin{itemize}
\item We revealed the existence of HUA in the learning process of ERN methods with theoretical analysis. The existence of HUA impedes the learning ability of evidential regression models, particularly in regions where ERN exhibits low confidence.
\item We propose a novel uncertainty regularization term designed to handle this HUA in evidential regression models and provide theoretical proof of its effectiveness.
\item Extensive experiments across multiple datasets and tasks are conducted to validate our theoretical findings and demonstrate the effectiveness of our proposed solution.
\end{itemize}

\section{Background}
\subsection{Problem Setup}
In the context of our study, we consider a regression task derived from a dataset $\mathcal{D}=\left\{\left(X_i, y_i\right)\right\}_{i=1}^N$, where $X_i \in \mathbb{R}^d$ denotes an independently and identically distributed (i.i.d.) input vector with $d$ dimensions. Corresponding to each input $X_i$, we have a real-valued target $y_i \in \mathbb{R}$. Our dataset comprises $N$ samples and the task is to predict the targets based on the input data points. We tackle the regression task by modeling the probabilistic distribution of the target variable $y$, which is formulated as $p\left(y \mid f_{\boldsymbol{\theta}}(X)\right)$, where $f$ refers to a neural network, and $\boldsymbol{\theta}$ denotes its parameters. For simplicity, we omit the subscript $i$.

\subsection{Evidential Regression Network}
As is illustrated in Figure~\ref{uncertainty_framework}, Evidential Regression Network (ERN)~\cite{NEURIPS2020_aab08546} introduces a Gaussian distribution $\mathcal{N}\left(\mu, \sigma^2\right)$ with unknown mean $\mu$ and variance $\sigma$ for modeling the regression problem. It is generally assumed that a target value $y$ is drawn i.i.d. from the Gaussian distribution, and that the unknown parameters $\mu$ and $\sigma$ follow a Normal Inverse-Gamma (NIG) distribution:
\begin{equation} \label{eq:1}
\begin{gathered}
y \sim \mathcal{N}\left(\mu, \sigma^2\right) \\
\mu \sim \mathcal{N}\left(\gamma, \sigma^2 v^{-1}\right) \quad \sigma^2 \sim \Gamma^{-1}(\alpha, \beta) \\
(\mu, \sigma^2) \sim \operatorname{NIG}(\gamma, v, \alpha, \beta)
\end{gathered}
\end{equation}
where $\Gamma(\cdot)$ is the gamma function, parameters $\boldsymbol{m}=(\gamma, v, \alpha, \beta)$, and $\gamma \in \mathbb{R}$, $v>0$, $\alpha>1$, $\beta>0$. The parameters of NIG distribution $\boldsymbol{m}$ is modeled by the output of a neural network $f_{\boldsymbol{\theta}}(\cdot)$, where $\boldsymbol{\theta}$ is the trainable parameters of such neural network. To enforce constraints on $(v, \alpha, \beta)$, a $\operatorname{SoftPlus}$ activation is applied (additional +1 added to $\alpha$). Linear activation is used for $\gamma \in \mathbb{R}$.
Considering the NIG distribution in Eq~\ref{eq:1}, the prediction, aleatoric uncertainty, and epistemic uncertainty can be calculated as the following:
\begin{equation}
\begin{gathered}
\underbrace{\mathbb{E}[\mu]=\gamma}_{\text {prediction }} 
\quad \underbrace{\mathbb{E}\left[\sigma^2\right]=\frac{\beta}{\alpha-1}}_{\text {aleatoric }} \quad  \underbrace{\operatorname{Var}[\mu]=\frac{\beta}{v(\alpha-1)}}_{\text {epistemic }} 
\end{gathered}
\end{equation}
Therefore, we can use $\mathbb{E}[\mu]=\gamma$ as the prediction of ERN, $\mathbb{E}\left[\sigma^2\right]=\frac{\beta}{\alpha-1}$ and $\operatorname{Var}[\mu]=\frac{\beta}{v(\alpha-1)}$ as the uncertainty estimation of ERN.

The likelihood of an observation $y$ given $\boldsymbol{m}$ is computed by marginalizing over $\mu$ and $\sigma^2$:
\begin{equation}
\label{eq:ern-prob}
%\begin{aligned}
p\left(y \mid \boldsymbol{m}\right) 
% &=\int_{\sigma^2=0}^{\infty} \int_{\mu=-\infty}^{\infty} p\left(y \mid \mu, \sigma^2\right) p\left(\mu, \sigma^2 \mid \boldsymbol{m}\right) \mathrm{d} \mu \mathrm{d} \sigma^2 \\
=\mathrm{St}\left(y ; \gamma, \frac{\beta(1+v)}{v \alpha}, 2 \alpha\right)
%\end{aligned}
\end{equation}
where $\mathrm{St}\left(y ; \mu_{\mathrm{St}}, \sigma_{\mathrm{St}}^2, v_{S t}\right)$ is the Student-t distribution with location $\mu_{\mathrm{St}}$, scale $\sigma_{\mathrm{St}}^2$ and degrees of freedom $v_{S t}$.

\subsubsection{Training Objective of ERN}
The parameters $\boldsymbol{\theta}$ of ERN are trained by maximizing the marginal likelihood in Eq.~\ref{eq:ern-prob}. The training objective is to minimize the negative logarithm of $p\left(y \mid \boldsymbol{m}\right)$, therefore the negative log likelihood (NLL) loss function is formulated as:
\begin{equation}
\begin{aligned}
\mathcal{L}_{{\theta}}^{\mathrm{NLL}}&=\frac{1}{2} \log \left(\frac{\pi}{v}\right)-\alpha \log (\Omega) \\
&+\left(\alpha+\frac{1}{2}\right) \log \left(\left(y-\gamma\right)^2 v +\Omega\right) \\
&+\log \left(\frac{\Gamma(\alpha)}{\Gamma\left(\alpha+\frac{1}{2}\right)}\right)
\end{aligned}
\end{equation}
where $\Omega=2 \beta(1+v)$.

To minimize the evidence on errors, the regularization term $\mathcal{L}_{\theta}^{\mathrm{R}} = \left|y-\gamma\right| \cdot(2 v+\alpha)$ is proposed to minimize evidence on incorrect predictions. Therefore, the loss function of ERN is:
\begin{equation}
\label{eq:ern-loss}
    \mathcal{L}_{\theta}^{\mathrm{ERN}} = \mathcal{L}_{{\theta}}^{\mathrm{NLL}} + \lambda \mathcal{L}_{\theta}^{\mathrm{R}}
\end{equation}
where $\lambda$ is a settable hyperparameter. For simplicity, we omit $\theta$ in the following sections.

\subsection{Variants of ERN}
\label{sec:back:variant}
ERN is for univariate regression and has been extended to multivariate regression with a different prior distribution normal-inverse-Wishart (NIW) distribution~\cite{meinert2021multivariate}. Multivariate ERN employs an NIW distribution and, similar to ERN, formulates the loss function as (see~\citeauthor{meinert2021multivariate} for details):
\begin{equation}
\begin{aligned}
\mathcal{L}^{\mathrm{MERN}} & \equiv  \mathcal{L}^{\mathrm{NLL}}=  \log \Gamma\left(\frac{\nu-n+1}{2}\right)-\log \Gamma\left(\frac{\nu+1}{2}\right) \\
& +\frac{n}{2} \log \left(r+\nu\right)-\nu \sum_j \ell_j \\
& +\frac{\nu+1}{2} \log \left|\boldsymbol{L} \boldsymbol{L}^{\top}+\frac{1}{r+\nu}\left(\vec{y}-\vec{\mu}_{0}\right)\left(\vec{y}-\vec{\mu}_{0}\right)^{\top}\right| \\
& +\text { const. }
\end{aligned}
\end{equation}

And estimation of the prediction as well as uncertainties as:
\begin{equation}
\begin{gathered}
\underbrace{\mathbb{E}[\mu]=\vec{\mu}_{0}}_{\text {prediction }} \quad 
\underbrace{\mathbb{E}[\boldsymbol{\Sigma}] \propto \frac{\nu}{\nu-n-1} \boldsymbol{L} \boldsymbol{L}^{\top}}_{\text {aleatoric }}  \\
\underbrace{\operatorname{var}[\vec{\mu}] \propto \mathbb{E}[\boldsymbol{\Sigma}] / \nu}_{\text {epistemic }}
\end{gathered}
\end{equation}

To learn the parameters $\boldsymbol{m}=\left(\vec{\mu}_{0}, \vec{\ell}, \nu\right)$, a NN has to have  $n(n+3) / 2+1$ outputs ($p_{1} \cdots p_{m}$). Also, activation functions have to be applied to the outputs of NN to ensure the following:
\begin{equation}
    \nu = n(n+5)/2+\tanh p_{\nu} \cdot  n(n+3)/2  + 1 > n+1
\end{equation}
where $p_{\nu} \in (p_{1} \cdots p_{m})$. And
\begin{equation}
\left(\boldsymbol{L}\right)_{j k}= \begin{cases}\exp \left\{\ell_j\right\} & \text { if } j=k \\ \ell_{j k} & \text { if } j>k \\ 0 & \text { else. }\end{cases}
\end{equation}
where $\ell_j, \ell_{j k} \in (p_{1} \cdots p_{m})$.

\section{Methodology}
In this section, we first give a definition of the High Uncertainty Area (HUA). Then, we theoretically analyze the existing limitation of ERN in HUA. Based on our analysis, we propose a novel solution to the problem. Finally, we extend our analysis and propose solutions to variants of ERN with other prior distributions.

\subsection{High Uncertainty Area (HUA) of ERN}
In this section, we show that in the high uncertainty area of ERN, the gradient of ERN will shrink to zero, therefore the outputs of ERN cannot be correctly updated. In this paper, we only study the gradient with respect to $\alpha$ as the gradient with respect to $v$ and $\beta$ follows a similar fashion.
\begin{definition}[\textbf{High Uncertainty Area}]
\textit{High Uncertainty Area is where $\alpha$ is close to 1, leading to very high uncertainty prediction.} 
\end{definition}
% We use $\boldsymbol{o}=(o_{\gamma}, o_{v}, o_{\alpha}, o_{\beta})$ to denote the output of $f_{\boldsymbol{\theta}}(X)$, therefore:
% \begin{equation}
% \begin{aligned}
%     \alpha&=\operatorname{SoftPlus}(o_{\alpha}) + 1 \Rightarrow \alpha = \log(\exp(o_{\alpha})+1)+1 \\
%     v&=\operatorname{SoftPlus}(o_v)  \Rightarrow v = \log(\exp(o_v)+1) \\
%     \beta&=\operatorname{SoftPlus}(o_{\beta})  \Rightarrow \beta = \log(\exp(o_{\beta})+1)
% \end{aligned}
% \end{equation}
% Where $\operatorname{SoftPlus}(\cdot)$ denotes $\operatorname{SoftPlus}$ activation. For a sample $X_i$, zero gradient area with respect to $\alpha$ is thus defined where:
% \begin{equation}
%     \frac{\partial \mathcal{L}^{\mathrm{ERN}}}{\partial o_\alpha} = 0
% \end{equation}
% And zero gradient area with respect to $v$ and $\beta$ follows a similar fashion.

% For samples in zero gradient area, the gradient becomes zero, therefore the output cannot be updated. Under such circumstances, the model cannot learn anything from the samples.
An effective model ought to possess the capacity to learn from the entire training samples. Unfortunately, this does not hold true in the context of ERN.
\begin{theorem}
\label{proof1}
\textit{ERN cannot learn from samples in high uncertainty area.}
\end{theorem}

\begin{proof}
Consider input $X$ and the corresponding label $y$. We use $\boldsymbol{o}=(o_{\gamma}, o_{v}, o_{\alpha}, o_{\beta})$ to denote the output of $f_{\boldsymbol{\theta}}(X)$, therefore:
\begin{equation}
\begin{aligned}
\alpha&=\operatorname{SoftPlus}(o_{\alpha}) + 1  \\
      &= \log(\exp(o_{\alpha})+1)+1 
\end{aligned}
\end{equation}
Where $\operatorname{SoftPlus}(\cdot)$ denotes $\operatorname{SoftPlus}$ activation (our theorem still holds true when faced with other popular activation functions, such as $\operatorname{ReLU}$ and $\operatorname{exp}$, see Appendix~\ref{appendix_1} for additional proofs). 
So the gradient of NLL loss with respect to $o_{\alpha}$ is given by:
\begin{equation}
\begin{aligned}
\frac{\partial \mathcal{L}^{\mathrm{NLL}}}{\partial o_\alpha} &= \frac{\partial \mathcal{L}^{\mathrm{NLL}}}{\partial \alpha} \frac{\partial \alpha}{\partial o_{\alpha}}     \\
&=[\log(1+\frac{\nu(\gamma-y)^2}{2\beta(\nu+1)}) +\psi(\alpha) \\
&- \psi(\alpha+0.5)] \cdot \operatorname{Sigmoid}\left(o_\alpha\right)
\end{aligned}
\end{equation}
where $\psi(\cdot)$ denotes the digamma function.

For a sample in high uncertainty area, we have:
\begin{equation}
    \alpha \rightarrow 1 \Rightarrow o_\alpha \rightarrow-\infty \Rightarrow \operatorname{Sigmoid}\left(o_\alpha\right) \rightarrow 0
\end{equation}
So, for such training samples:
\begin{equation}
    \frac{\partial \mathcal{L}^{\mathrm{NLL}}}{\partial o_\alpha} = 0
\end{equation}

And the gradient of $\mathcal{L}^{\mathrm{R}}=|y-\gamma| \cdot(2 v+\alpha)$ with respect to $o_{\alpha}$ is given by:
\begin{equation}
\begin{aligned}
    \frac{\partial \mathcal{L}^{\mathrm{R}}}{\partial o_\alpha}&=\frac{\partial \mathcal{L}^{\mathrm{R}}}{\partial \alpha} \frac{\partial \alpha}{\partial o_\alpha} \\
    &= |y-\gamma| \cdot \operatorname{Sigmoid}\left(o_\alpha\right)
\end{aligned}
\end{equation}
Similarly, we have:
\begin{equation}
    \frac{\partial \mathcal{L}^{\mathrm{R}}}{\partial o_\alpha} = 0
\end{equation}
And $\mathcal{L}^{\mathrm{ERN}}=\mathcal{L}^{\mathrm{NLL}}+\lambda \mathcal{L}^{\mathrm{R}}$, therefore we have:
\begin{equation}
\begin{aligned}
\frac{\partial \mathcal{L}^{\mathrm{ERN}}}{\partial o_\alpha} &= \frac{\partial \mathcal{L}^{\mathrm{NLL}}}{\partial o_\alpha} +\lambda \frac{\partial \mathcal{L}^{\mathrm{R}}}{\partial o_\alpha} \\
&=0
\end{aligned}
\end{equation}
\end{proof}

Since the gradient of the loss function with respect to $o_{\alpha}$ is zero, there won't be any update on $\alpha$ from such samples. The model fails to learn from samples in high uncertainty area.

\subsection{Uncertainty Regularization To Bypass HUA}
Considering the learning deficiency of ERN, in this paper, we propose an uncertainty regularization to solve the zero gradient problem within HUA:
\begin{equation}
    \mathcal{L}^{\mathrm{U}} =  -| y-\gamma | \cdot  \log(\exp(\alpha-1)-1) 
\end{equation}
% where $\lambda_{U}$ is a settable hyperparameter.
In this section, we show that $\mathcal{L}^{\mathrm{U}}$ can address the learning deficiency of ERN.
\begin{theorem}
\textit{Our proposed uncertainty regularization $\mathcal{L}^{\mathrm{U}}$ can learn from samples within HUA.}
\end{theorem}

\begin{proof}
The gradient of proposed regularization term $\mathcal{L}^{\mathrm{U}}$ with respect to $o_{\alpha}$ is given by:
\begin{equation}
\begin{aligned}
\frac{\partial \mathcal{L}^{\mathrm{U}}}{\partial o_\alpha} &= \frac{\partial \mathcal{L}^{\mathrm{U}}}{\partial \alpha} \frac{\partial \alpha}{\partial o_{\alpha}}     \\
&=-| y-\gamma | \cdot \frac{\exp (\alpha-1)}{\exp (\alpha-1)-1} \cdot \operatorname{Sigmoid}\left(o_\alpha\right) \\
&=-| y-\gamma | \cdot \left[1+\exp \left(-o_\alpha\right)\right] \cdot \frac{1}{1+\exp \left(-o_\alpha\right)} \\
&=-| y-\gamma |
\end{aligned}
\end{equation}
\end{proof}
Uncertainty regularization term $\mathcal{L}^{\mathrm{U}}$ ensures the maintenance of the gradient within the high uncertainty area. Importantly, the value of this gradient scales in accordance with the distance between the predicted value and the ground truth.

\begin{figure}
\centering
\includegraphics[width=0.9\columnwidth]{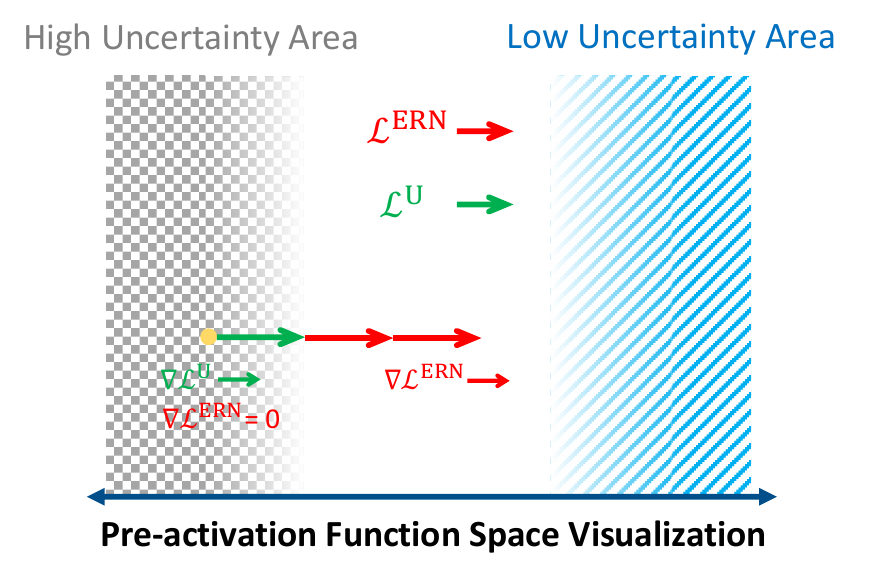} % Reduce the figure size so that it is slightly narrower than the column. Don't use precise values for figure width.This setup will avoid overfull boxes.
\caption{$\mathcal{L}^{\mathrm{ERN}}$ in Equation~\eqref{eq:ern-loss} cannot help the model get out of high uncertainty area while our proposed $\mathcal{L}^{\mathrm{U}}$ can still learn from samples in the grey area.}
\label{uncertainty_visualization}
\end{figure}

\subsubsection{Training of Regularized ERN}
The final training objective for the proposed Uncertainty Regularized ERN (UR-ERN) is formulated as:
\begin{equation}
    \mathcal{L} = \mathcal{L}^{\mathrm{ERN}} +\lambda_{1} \mathcal{L}^{\mathrm{U}}
\end{equation}
where $\lambda_{1}$ is a settable hyperparameter that balances the regularization and the original ERN loss. $\mathcal{L}^{\mathrm{NLL}}$ is for fitting purpose, $\mathcal{L}^{\mathrm{R}}$ regularizes evidence~\cite{NEURIPS2020_aab08546}. And our proposed $\mathcal{L}^{\mathrm{U}}$ addresses zero gradient problem in the HUA. %high uncertainty area.

\subsubsection{Uncertainty Space Visualization}
Figure~\ref{uncertainty_visualization} visualizes the uncertainty space with $x$-axis representing $o_{\alpha}$. Under ideal conditions, both fitting loss and uncertainty should be low, resulting in samples being mapped to the blue zone. Nevertheless, there exist certain samples predicted with high uncertainty, which may land within the grey region. Within this grey region, $\mathcal{L}^{\mathrm{ERN}}$ fails to update the parameters effectively. Under such circumstances, our proposed uncertainty regularization term $\mathcal{L}^{\mathrm{U}}$ retains the capacity to update the model. This enables the samples to be extracted from the grey area, thus allowing the training to continue.

\subsection{Uncertainty Regularization for ERN Variants}
Based on our theoretical analysis in previous sections, it is quite clear that the zero gradient problem in the HUA of ERN is attributable to certain activation functions that ensure non-negative values. Consequently, this limitation is not confined to ERN but can also extend to other evidential models that utilize similar activation functions. Multivariate ERN~\cite{meinert2021multivariate}, which we introduced in Section~\ref{sec:back:variant}, serves as an illustrative example; it suffers from similar problems to ERN, even when employing different prior distributions. Similar to the previous analysis, we study the parameter $\nu$ as an example. 

\begin{theorem}
\textit{Multivariate ERN~\cite{meinert2021multivariate} also cannot learn from samples in high uncertainty area.}
\end{theorem}

\begin{proof}
Given the output of a NN $\left(p_1 \cdots p_m\right)$, we have $p_\nu \in\left(p_1 \cdots p_m\right)$. And $\nu$ is formulated as the following:
\begin{equation}
    \nu = n(n+5)/2+\tanh p_{\nu} \cdot  n(n+3)/2  + 1
\end{equation}
Therefore, the gradient of loss function $\mathcal{L}^{\mathrm{MERN}}$ ($\mathcal{L}^{\mathrm{NLL}}$) with respect to $p_{\nu}$ is given by:
\begin{equation}
\begin{aligned}
\frac{\partial \mathcal{L}^{\mathrm{MERN}}}{\partial p_{\nu}}&=\frac{\partial \mathcal{L}^{\mathrm{NLL}}}{\partial \nu} \frac{\partial \nu}{\partial p_{\nu}} \\
&=\frac{\partial \mathcal{L}^{\mathrm{NLL}}}{\partial \nu} \cdot \frac{n(n+3)}{2} \cdot (1-\tanh^2 p_{\nu})
\end{aligned}
\end{equation}
within HUA, we have:
\begin{equation}
        p_{\nu} \rightarrow -\infty \Rightarrow (1-\tanh^2 p_{\nu}) \rightarrow 0 
\end{equation}
Therefore, we have:
\begin{equation}
    \frac{\partial \mathcal{L}^{\mathrm{MERN}}}{\partial p_{\nu}} = 0
\end{equation}

The gradient with respect to $p_{\nu}$ is zero, there will be no update to $p_{\nu}$. Multivariate ERN cannot learn effectively within HUA.
\end{proof}

Similarly, we propose uncertainty regularization term $\mathcal{L}^{\mathrm{U}}$ to help Multivariate ERN learn from samples within HUA. Since Multivariate ERN uses a different activation function, the proposed $\mathcal{L}^{\mathrm{U}}$ for Multivariate ERN has a different formulation:
\begin{equation}
    \mathcal{L}^{\mathrm{U}} = -\frac{1}{2} \cdot|y-\gamma| \cdot \log(\frac{n^2 + 3n}{n^2 + 4n+1-\nu}-1)
\end{equation}
We can prove the effectiveness of the proposed $\mathcal{L}^{\mathrm{U}}$ for Multivariate ERN.

\begin{theorem}
\textit{Our proposed uncertainty regularization $\mathcal{L}^{\mathrm{U}}$ can help Multivariate ERN learn from samples within HUA.}
\end{theorem}

\begin{proof}
The gradient of the proposed $\mathcal{L}^{\mathrm{U}}$ with respect to $p_{\nu}$ is given by:
\begin{equation}
\begin{aligned}
\frac{\partial \mathcal{L}^{\mathrm{U}}}{\partial p_{\nu}} &= \frac{\partial \mathcal{L}^{\mathrm{U}}}{\partial \nu} \frac{\partial \nu}{\partial p_{\nu}}     \\
&=-| y-\gamma | \cdot \frac{1}{2(n^2+3n)} \\
&\cdot \frac{-(n^2 + 3n)^2}{(\nu-n-1)(\nu-n^2-4n-1)} \\
&\cdot \frac{n^2+3n}{2} \cdot (1-\tanh^2 p_{\nu}) \\
&=-| y-\gamma | \cdot  \frac{1}{2(n^2+3n)} \\
&\cdot \frac{-(n^2 + 3n)^2}{(\nu-n-1)(\nu-n^2-4n-1)}\\
&\cdot \frac{n^2+3n}{2} \cdot \frac{-4(\nu-n-1)(\nu-n^2-4n-1)}{(n^2 + 3n)^2} \\
&=-| y-\gamma | 
\end{aligned}
\end{equation}

\end{proof}

The proposed regularization term $\mathcal{L}^{\mathrm{U}}$ guarantees a non-zero gradient for Multivariate ERN in the HUA. Therefore, the loss function for uncertainty regularized Multivariate ERN is formulated as:
\begin{equation}
    \mathcal{L}=\mathcal{L}^{\mathrm{NLL}}+\lambda_1 \mathcal{L}^{\mathrm{U}}
\end{equation}
where $\lambda_1$ is settable hyperparameter.

While the two terms have different formulations, they share a common intuition. We identify the zero gradient problem arising from the activation function and introduce a term to circumvent zero gradients, simultaneously increasing $\alpha$.
%avoid zero gradients and increase $\alpha$.
This adjustment guides the training process away from this problematic area.
%guiding training away from this problematic area. 
Our mathematical analysis confirms these terms effectively achieve our objective.

The above theoretical analysis reveals that the learning deficiency is not exclusive to ERN~\cite{NEURIPS2020_aab08546}; it also manifests in other evidential models~\cite{meinert2021multivariate} that employ different prior distributions.

\section{Experiments}
In this section, we first conduct experiments under both synthetic and real-world datasets. For each dataset, we investigate whether the methods fail to learn from samples within and outside HUA. Moreover, we perform additional experiments to demonstrate that even the Multivariate ERN, which employs distinct prior distributions, struggles to learn effectively within HUA. To compare performance, we use baselines including ERN~\cite{NEURIPS2020_aab08546} ($\mathcal{L}^{\mathrm{NLL}}+\lambda \mathcal{L}^{\mathrm{R}}$), and NLL-ERN ($\mathcal{L}^{\mathrm{NLL}}$). For experiments within HUA, we initialize the model within HUA by setting bias in the activation layer. %(See details in Appendix \ref{appendix_2}). 
Please refer to Appendix~\ref{appendix_2} for details about experimental setups and experiments about the sensitivity of hyperparameters.
%are shown in Appendix~\ref{appendix_2}.

\subsection{Performance on Cubic Regression Dataset}
To highlight the limitations of ERN, we compare its performance with our proposed \ours on cubic regression dataset~\cite{NEURIPS2020_aab08546} within HUA. Following~\cite{NEURIPS2020_aab08546}, we train models on $y=x^3+\epsilon$, where $\epsilon \sim \mathcal{N}(0,3)$. We conduct training over the interval $x \in[-4,4]$, and perform testing over $x \in \left[-6,-4) \cup (4,6\right]$.

\begin{figure}[t!]
\centering
\includegraphics[width=0.9\columnwidth]{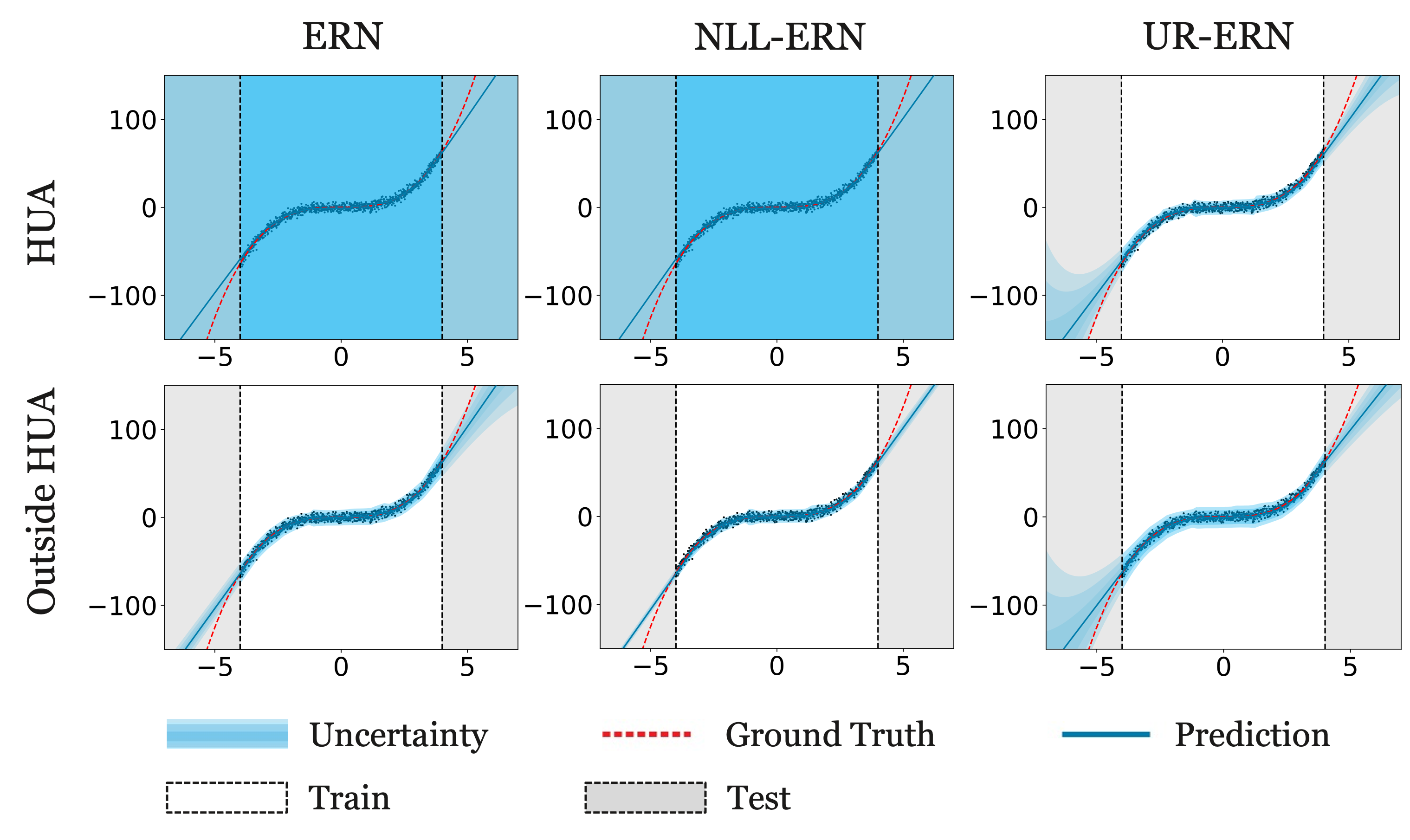} % Reduce the figure size so that it is slightly narrower than the column. Don't use precise values for figure width.This setup will avoid overfull boxes.
\caption{Uncertainty estimation on Cubic Regression. The blue shade represents prediction uncertainty. An effective evidential model would cause the blue shade to cover the distance between the predicted value and the ground truth precisely. Up: Comparison of model performance within HUA. Down: Comparison of model performance outside HUA. \ours can cover the ground truth precisely under both within HUA and outside HUA.}
\label{fig:toy_dataset}
\end{figure}

% As demonstrated in Figure~\ref{fig:toy_dataset}, the blue shade represents uncertainty predicted by the models respectively. ERN struggles to update parameters within HUA, leading to high uncertainty prediction across the dataset, whereas \ours retains its training efficiency. This finding validates our previous theoretical analysis.
\subsubsection{Evaluation Metrics}
% Our proposed regularization mostly helps the model update $\alpha$ effectively within HUA, and the value of $\alpha$ is directly reflected in uncertainty prediction. Based on previous theoretical analysis, if the model fails to update $\alpha$ correctly, the uncertainty prediction will be unreasonably high.
% Therefore, we use uncertainty prediction as the evaluation metric. As is shown in~\ref{fig:toy_dataset}, it is represented as blue shade. For cubic regression, a better uncertainty prediction should be that the blue shade covers the the distance of predicted value and ground truth.
Our proposed regularization is mainly designed to help the model effectively update the parameter $\alpha$ within HUA. This is essential because, as our theoretical analysis has shown, if the model cannot properly update $\alpha$, the uncertainty prediction will become unreasonably high. Therefore, we have chosen uncertainty prediction as our evaluation metric. We visualize the experimental results of uncertainty estimation along with ground truth in Figure~\ref{fig:toy_dataset} and the uncertainty is represented by the blue shade. An accurate prediction in uncertainty would lead the blue shade to cover the distance between the predicted value and the ground truth precisely.

\subsubsection{Cubic Regression within HUA}
As illustrated in Figure~\ref{fig:toy_dataset}, where the blue shade represents the uncertainty predicted by the models, ERN encounters difficulties in updating parameters in the HUA, resulting in high uncertainty predictions across the dataset. In contrast, the proposed \ours maintains its training efficiency, effectively mitigating this issue. These observations validate our theoretical analysis, demonstrating the effectiveness of our proposed method. 
% We further explore the performance of these methods under normal settings (not within HUA). Figure~\ref{fig:toy_dataset_no_HUA} clearly shows the $\mathcal{L}^{\mathrm{R}}$ in $\mathcal{L}^{\mathrm{ERN}}$ helps the model get better uncertainty predictions, which our experiments validate the findings in~\citeauthor{NEURIPS2020_aab08546}. It can also be seen that the proposed \ours can perform well outside HUA, even better than ERN. This again demonstrates the effectiveness of our method.
\subsubsection{Cubic Regression outside HUA}
We extend our investigation to assess the performance of these methods under standard conditions (outside the HUA). Figure~\ref{fig:toy_dataset} illustrates that the inclusion of the term $\mathcal{L}^{\mathrm{R}}$ in $\mathcal{L}^{\mathrm{ERN}}$ contributes to more accurate uncertainty predictions, a result that aligns with the findings in~\citeauthor{NEURIPS2020_aab08546}. Moreover, the proposed \ours not only performs robustly in the HUA but also exhibits superior performance compared to the ERN outside the HUA. These observations further demonstrate the effectiveness of our method.

\begin{figure*}[t!]
  \centering
  \subfloat[Performance of model within HUA of Depth Estimation]{
    \includegraphics[width=0.5\linewidth]{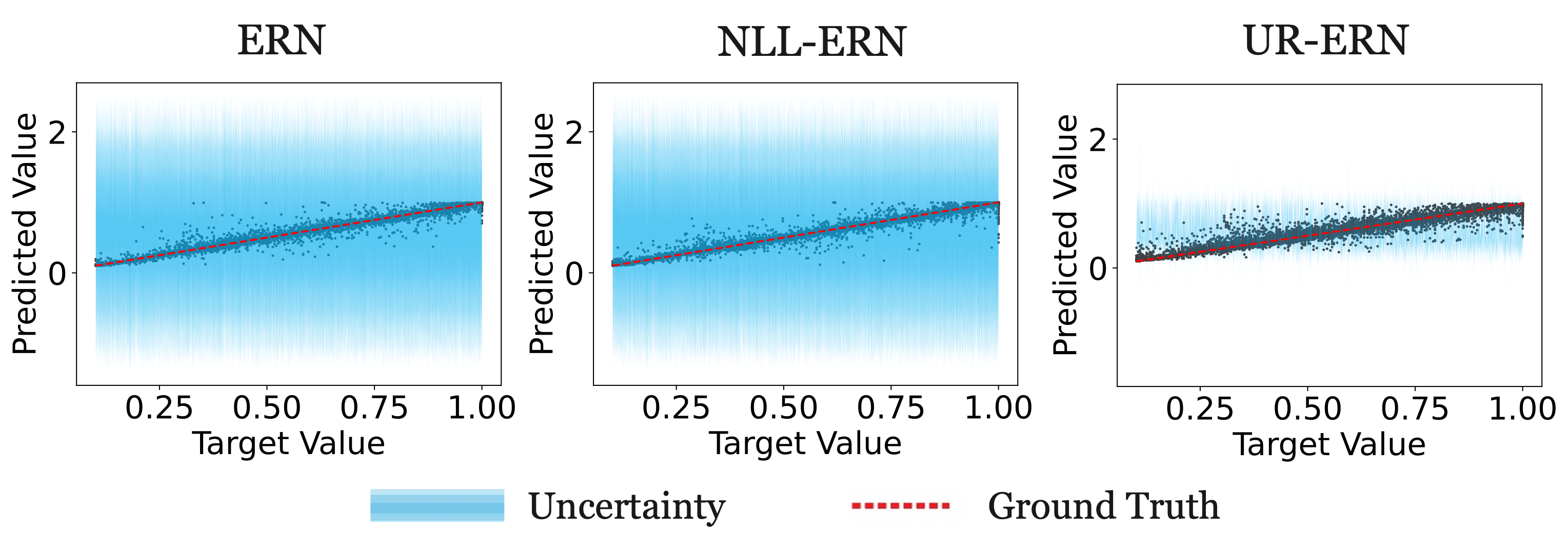}
  }
  \subfloat[RMSE with confidence]{
    \includegraphics[width=0.22\linewidth]{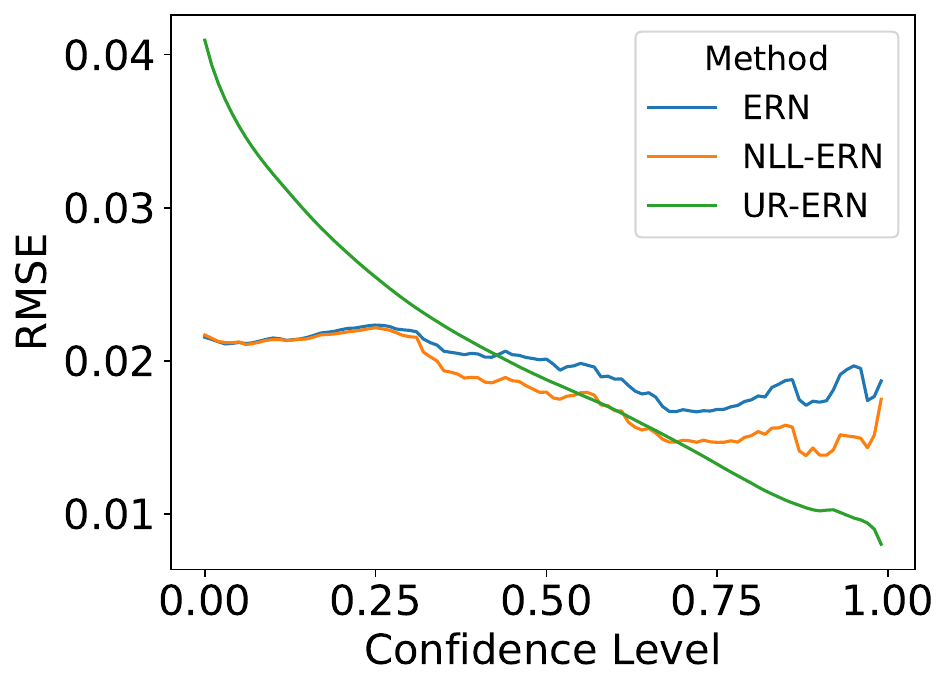}
  }
  \subfloat[Uncertainty calibration]{
    \includegraphics[width=0.22\linewidth]{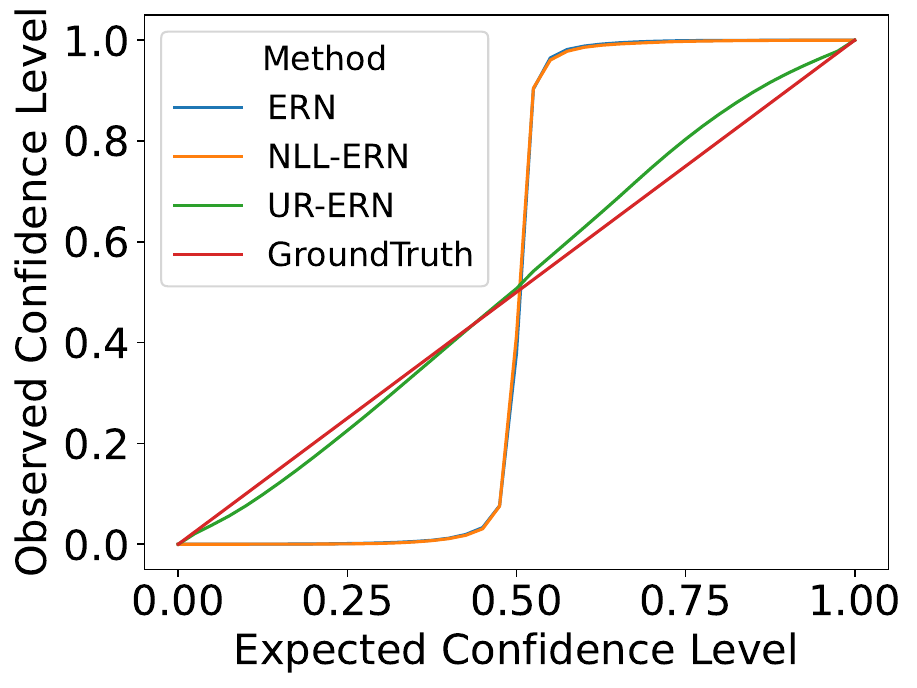}
  }
  \caption{Uncertainty prediction of Depth Estimation within HUA. (a) The blue shade represents prediction uncertainty. A good estimation of uncertainty should cover the gap between prediction and ground truth exactly. (b) Root Mean Square Error (RMSE) at various confidence levels. The evidential model with a larger confidence level should have a lower RMSE. (c) Uncertainty calibration calculated following~\citeauthor{kuleshov2018accurate}, the ideal curve is $y=x$. The calibration errors are 0.2261, 0.2250, and 0.0243 for ERN, NLL-ERN and UR-ERN, respectively.}
  \label{fig:depth_HUA}
\end{figure*}

\begin{figure*}[t!]
  \centering
  \includegraphics[width=0.95\linewidth]{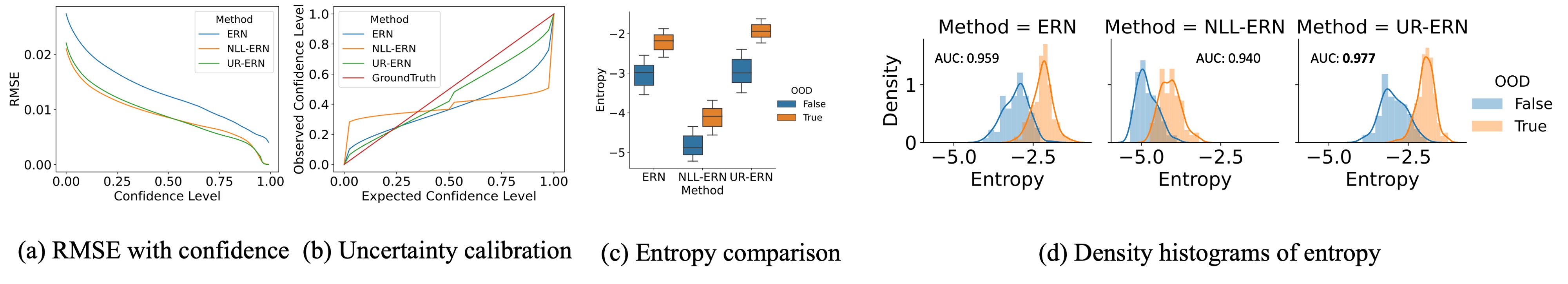} 
  \caption{Uncertainty prediction of Depth Estimation outside HUA. (a) RMSE at various confidence levels. (b) Uncertainty calibration (ideal: $y=x$). The calibration errors are 0.1366, 0.1978, and 0.0289 for ERN, NLL-ERN and UR-ERN, respectively. (c) and (d) show OOD experimental results. (c) Entropy comparisons for different methods. (d) Density histograms of entropy. Entropy is calculated from $\sigma$, directly related to uncertainty. A good evidential model should be able to distinguish OOD data.}
  \label{fig:depth_outside_HUA}
\end{figure*}

\subsection{Performance on Monocular Depth Estimation}

We further evaluate the performance of our proposed \ours and ERN on more challenging real-world tasks. Monocular depth estimation is a task in computer vision aiming to predict the depth directly from an RGB image. We choose the NYU Depth v2 dataset~\cite{silberman2012indoor} for experiments.
For each pixel, there is a corresponding depth target. Following previous practice~\cite{NEURIPS2020_aab08546}, we train U-Net \cite{ronneberger2015u} style neural network as the backbone to learn evidential parameters. Similar to the previous section, we compare the performance of our \ours against ERN within HUA and outside HUA. Limited by space, additional experimental results are summarized in Appendix~\ref{appendix_2}.

\subsubsection{Evaluation Metrics}
For uncertainty in depth estimation, we first explore whether the models can correctly update parameters within HUA. Similarly, we choose the value of uncertainty as the evaluation metric. Similar to cubic regression, the blue shade in Figure~\ref{fig:depth_HUA}(a) depicts  predicted uncertainty. 
And the models that cannot learn from samples within HUA will exhibit excessively large blue shade areas, resulting from their high uncertainty prediction across the test set. 
%Models that cannot learn from samples within HUA will exhibit excessively large areas of blue shading, indicating high uncertainty predictions throughout the test set.

Following existing works~\citeauthor{NEURIPS2020_aab08546,kuleshov2018accurate}, we also use cutoff curves and calibration curves to compare the performance of uncertainty estimation. Inspired by previous work~\cite{NEURIPS2020_aab08546}, we further test how the models perform when faced with OOD data. An effective evidential model should predict high uncertainty for OOD data and can distinguish the OOD data. The OOD experimental setup is the same as~\citeauthor{NEURIPS2020_aab08546} for comparison.

\subsubsection{Monocular Depth Estimation within HUA}
As illustrated in Figure~\ref{fig:depth_HUA}, ERN with $\mathcal{L}^{\mathrm{R}}$ or not, struggles to update parameters effectively within the HUA, leading to suboptimal uncertainty estimation. This constraint forms a significant impediment to effective learning from particular samples. In contrast, the proposed \ours successfully navigates this challenge, demonstrating the capacity to learn from these specific samples and to efficiently estimate uncertainty, mirroring the behavior observed in normal regions.

% \begin{figure}
%   \centering
%   \subfloat[]{
%     \includegraphics[width=0.45\linewidth]{depth_outside_HUA_1.pdf}
%   }
%   \subfloat[]{
%     \includegraphics[width=0.45\linewidth]{depth_outside_HUA_2.pdf}
%   }\\
%   \subfloat[]{
%     \includegraphics[width=0.5\linewidth]{depth_outside_HUA_3.pdf}
%   }\\
%   \subfloat[]{
%     \includegraphics[width=0.9\linewidth]{depth_outside_HUA_4.pdf}
%   }
%   \caption{Uncertainty prediction of depth estimation outside HUA. (a) RMSE at various confidence levels. (b) Uncertainty calibration (ideal: $y=x$). (c) and (d) show OOD experimental results. (c) Entropy comparisons for different methods. (d) Density histograms of entropy. Entropy is calculated from $\sigma$, directly related to uncertainty. A good evidential model should be able to distinguish OOD data.}
%   \label{fig:depth_outside_HUA}
% \end{figure}

% \begin{figure}
% \centering
% \includegraphics[width=0.9\columnwidth]{depth_HUA.png} % Reduce the figure size so that it is slightly narrower than the column. Don't use precise values for figure width.This setup will avoid overfull boxes.
% \caption{Uncertainty prediction of depth estimation within HUA.}
% \label{fig:depth_HUA}
% \end{figure}

Figure~\ref{fig:depth_HUA} presents a comparison of model performances as pixels possessing uncertainty beyond specific thresholds are excluded. The proposed \ours demonstrates robust behavior, characterized by a consistent reduction in error corresponding to increasing levels of confidence. In addition to the performance comparison, Figure~\ref{fig:depth_HUA} provides an assessment of the calibration of our uncertainty estimates. The calibration curves, computed following the methodology described in previous work~\cite{kuleshov2018accurate}, should ideally follow $y=x$ for accurate representation. The respective calibration errors for each model are also shown.

\subsubsection{Monocular Depth Estimation outside HUA}
We also look at how the models perform outside the HUA. Figure~\ref{fig:depth_outside_HUA} visualizes the comparison of how the models can estimate uncertainty in depth estimation outside HUA. The proposed \ours has a lower Root Mean Square Error (RMSE) for most confidence levels than the competing models. And the calibration curve is closer to the ideal curve $y=x$ than any competing model, showing again the effectiveness of the proposed \ours.

For OOD experiments, the proposed \ours can distinguish OOD data better than the competing models. The above experiments reveal that the proposed regularization can not only be effective at guiding the model to get out of HUA, but it also performs well outside HUA. 

% \begin{figure}
% \centering
% \includegraphics[width=0.9\columnwidth]{depth_outside_HUA.png} % Reduce the figure size so that it is slightly narrower than the column. Don't use precise values for figure width.This setup will avoid overfull boxes.
% \caption{Uncertainty prediction of depth estimation within HUA.}
% \label{fig:depth_outside_HUA}
% \end{figure}
% \begin{figure}
%   \centering
%   \subfloat[]{
%     \includegraphics[width=0.45\linewidth]{depth_outside_HUA_1.pdf}
%   }
%   \subfloat[]{
%     \includegraphics[width=0.45\linewidth]{depth_outside_HUA_2.pdf}
%   }\\
%   \subfloat[]{
%     \includegraphics[width=0.5\linewidth]{depth_outside_HUA_3.pdf}
%   }\\
%   \subfloat[]{
%     \includegraphics[width=0.9\linewidth]{depth_outside_HUA_4.pdf}
%   }
%   \caption{Uncertainty prediction of depth estimation outside HUA.}
%   \label{fig:depth_outside_HUA}
% \end{figure}

\subsection{Extension to Different ERN Variants}
Our theoretical findings reveal that the performance issues within HUA extend beyond ERN. Other evidential models, even those utilizing different prior distributions, similarly exhibit poor performance within this challenging region. 

Following the theoretical analysis in the previous section, we compare the performance of models in the context of Multivariate Deep Evidential Regression~\cite{meinert2021multivariate}. Following the experimental setup in~\cite{meinert2021multivariate}, we conduct the multivariate experiment and predict $(x, y) \in \mathbb{R}^2$ given $t \in \mathbb{R}$, where $x$ and $y$ being the features of the data sample given input $t$ with the following definition:
\begin{equation}
x=(1+\epsilon) \cos t, \quad 
y=(1+\epsilon) \sin t,
\end{equation}
and the distribution of $t$ is formulated as following:
\begin{equation}
t \sim \begin{cases}1-\frac{\zeta}{\pi} & \text { if } \zeta \in[0, \pi] \\ \frac{\zeta}{\pi}-1 & \text { if } \zeta \in(\pi, 2 \pi] \\ 0 & \text { else }\end{cases}
\end{equation}
where $\zeta \in[0,2 \pi]$ is uniformly distributed and $\epsilon \sim \mathcal{N}(0,0.1)$ is drawn from a normal distribution. Under this setting, the uncertainty is calculated as $\frac{\boldsymbol{L} \boldsymbol{L}^{\top}}{\nu-3}$. When $\nu\rightarrow3$, the corresponding uncertainty will be infinite across the dataset. Similar to previous experiments, we initialize the model within HUA by setting bias in the activation layer (See details in Appendix~\ref{appendix_2}).

% \begin{figure}
%     \centering
%     \begin{subfigure}{0.5\columnwidth}
%         \centering
%         \includegraphics[width=0.95\linewidth]{nu_MERN.png}
%         \caption{Multivariate ERN}
%         %\label{fig:image_cutoff}
%     \end{subfigure}%
%         \begin{subfigure}{0.5\columnwidth}
%         \centering
%         \includegraphics[width=0.95\linewidth]{nu_URERN.png}
%         \caption{\ours}
%         %\label{fig:image_calib}
%     \end{subfigure}
%     \caption{Prediction of parameter $\nu$ of Multivariate ERN and our proposed \ours. Low (high) values correspond to large (low) uncertainty prediction.}
%     \label{fig:multivariate}
% \end{figure}

Figure~\ref{fig:multivariate}(a) shows that the Multivariate ERN struggles to update the parameter $\nu$, resulting in unreasonably high uncertainty estimations (see Appendix~\ref{appendix_2} for additional experimental results). Consistent with previous sections, the proposed \ours in Figure~\ref{fig:multivariate}(b) does not encounter this issue, effectively learning from samples within HUA and providing reasonable uncertainty predictions. This validates our theoretical findings, demonstrating that evidential models, including but not limited to ERN, face challenges in the HUA when utilizing specific activation functions to ensure non-negative values. As a solution to this problem, our proposed uncertainty regularization term $\mathcal{L}^{\mathrm{U}}$ can be universally applied to these methods, enabling them to avoid the issue of zero gradients in the HUA.

\begin{figure}[h!]
    \centering
    \subfloat[Multivariate ERN]{
        \includegraphics[width=0.45\linewidth]{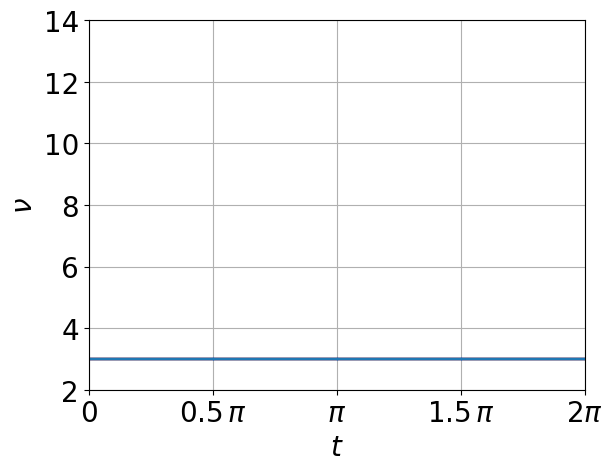}
        %\label{fig:image_cutoff}
    }
    \subfloat[\ours]{
        \includegraphics[width=0.45\linewidth]{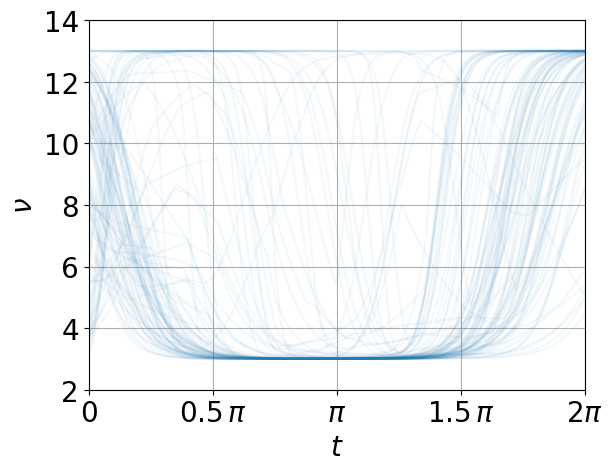}
        %\label{fig:image_calib}
    }
    \caption{Prediction of parameter $\nu$ in Multivariate ERN and our proposed \ours. Uncertainty ($\frac{\boldsymbol{L} \boldsymbol{L}^{\top}}{\nu-3}$) will be infinite if $\nu$ is close to 3, indicating the evidential model fails to properly estimate the uncertainty of predictions.}
    \label{fig:multivariate}
\end{figure}

% \begin{figure}
%     \centering
%     \subfloat[Multivariate ERN]{
%         \includegraphics[width=0.45\linewidth]{nu_MERN.png}
%         %\label{fig:image_cutoff}
%     }
%     \subfloat[\ours]{
%         \includegraphics[width=0.45\linewidth]{nu_URERN.png}
%         %\label{fig:image_calib}
%     }
%     \caption{Prediction of parameter $\nu$ of Multivariate ERN and our proposed \ours. Low (high) values correspond to large (low) uncertainty prediction.}
%     \label{fig:multivariate}
% \end{figure}
\section{Related Works}
\subsubsection{Uncertainty Estimation in Deep Learning} 
% For developing trustworthy Deep Learning (DL) model, a good estimation of prediction uncertainty is essential. Ensemble-based methods~\cite{pearce2020uncertainty,lakshminarayanan2017simple} and Bayesian neural networks~\cite{gal2016dropout,wilson2020bayesian,blundell2015weight} are commonly used to quantify predictive uncertainty. Ensemble-based methods train multiple neural networks and aggregate their predictions for uncertainty quantification. Since multiple models are required, ensemble-based methods need more parameters, which is computationally expensive for real-world applications. Alternatively, Bayesian neural networks (BNNs) treat the weights of the neural network as random variables and capture the distribution over the weights rather than point estimates.~\cite{gal2016dropout} introduced Dropout during the inference phase to obtain predictive uncertainty, approximating Bayesian inference in deep Gaussian processes. 
% %However, these methods also bring a significant increase in computational cost as they employ complex sampling methods.
% However, these methods also come with a notable rise in computational expense due to their use of intricate sampling techniques.
Developing a trustworthy Deep Learning (DL) model requires an accurate estimation of prediction uncertainty. Ensemble methods~\cite{pearce2020uncertainty,lakshminarayanan2017simple} use multiple networks for uncertainty quantification and thus are computationally expensive due to the need for more parameters. Bayesian neural networks (BNNs)~\cite{gal2016dropout,wilson2020bayesian,blundell2015weight}, treating neural network weights as random variables, capture weight distribution rather than point estimates. The introduction of Dropout to BNNs during inference~\cite{gal2016dropout} approximates Bayesian inference in deep Gaussian processes but also increases computational costs due to complex sampling techniques.

\subsubsection{Evidential Deep Learning}
% Evidential Deep Learning (EDL)~\cite{sensoy2018evidential,NEURIPS2020_aab08546,malinin2018predictive} is a relatively recent approach to uncertainty estimation in deep learning. Evidential models incorporate a conjugate higher-order evidential prior, empowering the model to grasp the fine-grained uncertainties. In Evidential models, the neural network is trained to predict parameters of the output distribution that not only capture the target variable but also the uncertainty associated with the prediction. For example, Dirichlet prior is introduced for evidential classification~\cite{sensoy2018evidential,bao2021evidential}, and a neural network is constructed to get the desired parameters of Dirichlet distribution. Similarly, NIG prior is introduced over the Gaussian likelihood for evidential regression~\cite{NEURIPS2020_aab08546,oh2022improving}.~\citeauthor{meinert2021multivariate} further utilize NIW prior for multivariate regression.
% In contrast to the above-mentioned methods, evidential models generally have lower computational overhead.
Evidential Deep Learning (EDL)~\cite{sensoy2018evidential,NEURIPS2020_aab08546,malinin2018predictive} is a relatively recent method for uncertainty estimation in deep learning, using a conjugate higher-order evidential prior to understand fine-grained uncertainties. These models train the neural network to predict distribution parameters that capture both the target variable and its associated uncertainty. Dirichlet prior is introduced for evidential classification~\cite{sensoy2018evidential}. And NIG prior is introduced for evidential regression~\cite{NEURIPS2020_aab08546}.~\citet{meinert2021multivariate} further utilize NIW prior for multivariate regression.~\citet{pandey2023learn} first observed convergence issues in evidential models for classification, noting their incapacity to learn from certain samples. To tackle this, they introduced novel evidence regularization. However, the convergence analysis of ERN for regression tasks remains unexplored.

\section{Conclusion}
% In this paper, we define High Uncertainty Area for evidential regression models and provide a theoretical analysis of a limitation of these models. We prove that the gradient of samples in HUA shrinks to zero which means evidential regression models cannot learn from these training samples. To tackle this problem, we propose a novel regularization term to guide ERN learning from all training samples. Our thorough experiments demonstrate the effectiveness of our proposed method.
% In this paper, we introduce the concept of High Uncertainty Area (HUA) for evidential regression models and conduct a theoretical analysis of a significant limitation within these models. Specifically, we prove that the gradient of samples within the HUA diminishes to zero, rendering the evidential regression models unable to learn from these training samples. To address this challenge, we propose a novel regularization term designed to guide the evidential regression models to learn from all available training samples. Our comprehensive experiments demonstrate the effectiveness of this innovative approach.
In this paper, we define High Uncertainty Area (HUA) for evidential regression models and identify a key limitation where the gradient of samples within the HUA diminishes to zero. This makes the evidential models unable to learn from these samples. To combat this issue, we introduce a novel regularization term, and our experiments validate the effectiveness of this solution.

\section{Acknowledgments}
This study is partially supported by NSF award (IIS 2045848, IIS 1837956, IIS 2319450, and IIS 2153311).

\bibliography{aaai23}
\clearpage
\appendix

\section{Additional proof for theorem~\ref{proof1}}
\label{appendix_1}
\textbf{Theorem~\ref{proof1}.} \textit{ERN cannot learn from samples in high uncertainty area.}

\begin{proof}
Consider input $X$ and the corresponding label $y$. We use $\boldsymbol{o}=(o_{\gamma}, o_{v}, o_{\alpha}, o_{\beta})$ to denote the output of $f_{\boldsymbol{\theta}}(X)$.

\textbf{Case I}: $\operatorname{ReLU}(\cdot)$ activation function to get $\alpha$, therefore we have:
\begin{equation}
\begin{aligned}
\alpha&=\operatorname{ReLU}(o_{\alpha}) + 1  \\
      &= \max(0,o_{\alpha})+1 
\end{aligned}
\end{equation}
And,
\begin{equation}
    \frac{\partial \alpha}{\partial o_{\alpha}}=
    \begin{cases}1 & \text { if } \quad o_{\alpha}>0 \\ 
    0 & \text { otherwise }
    \end{cases}
\end{equation}
So the gradient of NLL loss with respect to $o_{\alpha}$ is given by:
\begin{equation}
\begin{aligned}
\frac{\partial \mathcal{L}^{\mathrm{NLL}}}{\partial o_\alpha} &= \frac{\partial \mathcal{L}^{\mathrm{NLL}}}{\partial \alpha} \frac{\partial \alpha}{\partial o_{\alpha}}     \\
&=[\log(1+\frac{\nu(\gamma-y)^2}{2\beta(\nu+1)}) +\psi(\alpha) \\
&- \psi(\alpha+0.5)] \cdot \frac{\partial \alpha}{\partial o_{\alpha}}
\end{aligned}
\end{equation}
where $\psi(\cdot)$ denotes the digamma function.
For a sample in high uncertainty area, we have:
\begin{equation}
    \alpha \rightarrow 1 \Rightarrow o_\alpha <0 \Rightarrow \frac{\partial \alpha}{\partial o_{\alpha}}=0
\end{equation}

\textbf{Case II}: $\operatorname{exp}(\cdot)$ activation function to get $\alpha$, therefore we have:
\begin{equation}
\alpha=\operatorname{exp}(o_{\alpha}) + 1  
\end{equation}
And,
\begin{equation}
    \frac{\partial \alpha}{\partial o_{\alpha}}=\exp(o_{\alpha})
\end{equation}
So the gradient of NLL loss with respect to $o_{\alpha}$ is given by:
\begin{equation}
\begin{aligned}
\frac{\partial \mathcal{L}^{\mathrm{NLL}}}{\partial o_\alpha} &= \frac{\partial \mathcal{L}^{\mathrm{NLL}}}{\partial \alpha} \frac{\partial \alpha}{\partial o_{\alpha}}     \\
&=[\log(1+\frac{\nu(\gamma-y)^2}{2\beta(\nu+1)}) +\psi(\alpha) \\
&- \psi(\alpha+0.5)] \cdot \exp(o_{\alpha})
\end{aligned}
\end{equation}
where $\psi(\cdot)$ denotes the digamma function.

For a sample in high uncertainty area, we have:
\begin{equation}
    \alpha \rightarrow 1 \Rightarrow o_\alpha \rightarrow-\infty \Rightarrow \exp\left(o_\alpha\right) \rightarrow 0
\end{equation}

\textbf{Case III}: $\operatorname{SoftPlus}(\cdot)$ activation function to get $\alpha$, therefore we have:
\begin{equation}
\begin{aligned}
\alpha&=\operatorname{SoftPlus}(o_{\alpha}) + 1  \\
      &= \log(\exp(o_{\alpha})+1)+1 
\end{aligned}
\end{equation}
And,
\begin{equation}
    \frac{\partial \alpha}{\partial o_{\alpha}}=\operatorname{Sigmoid}\left(o_\alpha\right)
\end{equation}
So the gradient of NLL loss with respect to $o_{\alpha}$ is given by:
\begin{equation}
\begin{aligned}
\frac{\partial \mathcal{L}^{\mathrm{NLL}}}{\partial o_\alpha} &= \frac{\partial \mathcal{L}^{\mathrm{NLL}}}{\partial \alpha} \frac{\partial \alpha}{\partial o_{\alpha}}     \\
&=[\log(1+\frac{\nu(\gamma-y)^2}{2\beta(\nu+1)}) +\psi(\alpha) \\
&- \psi(\alpha+0.5)] \cdot \operatorname{Sigmoid}\left(o_\alpha\right)
\end{aligned}
\end{equation}
where $\psi(\cdot)$ denotes the digamma function.

For a sample in high uncertainty area, we have:
\begin{equation}
    \alpha \rightarrow 1 \Rightarrow o_\alpha \rightarrow-\infty \Rightarrow \operatorname{Sigmoid}\left(o_\alpha\right) \rightarrow 0
\end{equation}

Based on the above analysis of \textbf{Case I}, \textbf{Case II} and \textbf{Case III}, for training samples in high uncertainty areas:
\begin{equation}
    \frac{\partial \mathcal{L}^{\mathrm{NLL}}}{\partial o_\alpha} = 0
\end{equation}

And the gradient of $\mathcal{L}^{\mathrm{R}}=|y-\gamma| \cdot(2 v+\alpha)$ with respect to $o_{\alpha}$ is given by:
\begin{equation}
\begin{aligned}
    \frac{\partial \mathcal{L}^{\mathrm{R}}}{\partial o_\alpha}&=\frac{\partial \mathcal{L}^{\mathrm{R}}}{\partial \alpha} \frac{\partial \alpha}{\partial o_\alpha} \\
    &= |y-\gamma| \cdot \frac{\partial \alpha}{\partial o_\alpha}
\end{aligned}
\end{equation}
Similarly, we have:
\begin{equation}
    \frac{\partial \mathcal{L}^{\mathrm{R}}}{\partial o_\alpha} = 0
\end{equation}
And $\mathcal{L}^{\mathrm{ERN}}=\mathcal{L}^{\mathrm{NLL}}+\lambda \mathcal{L}^{\mathrm{R}}$, therefore we have:
\begin{equation}
\begin{aligned}
\frac{\partial \mathcal{L}^{\mathrm{ERN}}}{\partial o_\alpha} &= \frac{\partial \mathcal{L}^{\mathrm{NLL}}}{\partial o_\alpha} +\lambda \frac{\partial \mathcal{L}^{\mathrm{R}}}{\partial o_\alpha} \\
&=0
\end{aligned}
\end{equation}
\end{proof}

Since the gradient of the loss function with respect to $o_{\alpha}$ is zero, there won't be any update on $\alpha$ from such samples. The model fails to learn from samples in high uncertainty area.

\section{Details about experimental setup and additional experiments}
\label{appendix_2}

\subsection{Experimental Setup}
\label{appendix_21}
\subsubsection{Cubic Regression}For Cubic Regression, the problem setup is the same as~\citeauthor{NEURIPS2020_aab08546}.
The training set is composed of examples derived from the cubic equation $y=x^3+\epsilon$, where the error term $\epsilon$ follows a normal distribution $\epsilon \sim \mathcal{N}(0,3)$. We conduct training over the interval $x \in[-4,4]$, and perform testing over $x \in[-6,-4) \cup(4,6]$. All models consisted of 100 neurons with 3 hidden layers and were trained to convergence. The coefficient for $\mathcal{L}^{\mathrm{R}}$ in $\mathcal{L}^{\mathrm{ERN}}$ is $\lambda=0.01$. All models were trained with the Adam optimizer $\eta=5 \mathrm{e}-3$ and a batch size of 128. For experiments within HUA, we give a large negative bias to the activation layer to make the samples fall within HUA at the beginning of training.

\begin{figure}
\centering
\includegraphics[width=0.9\columnwidth]{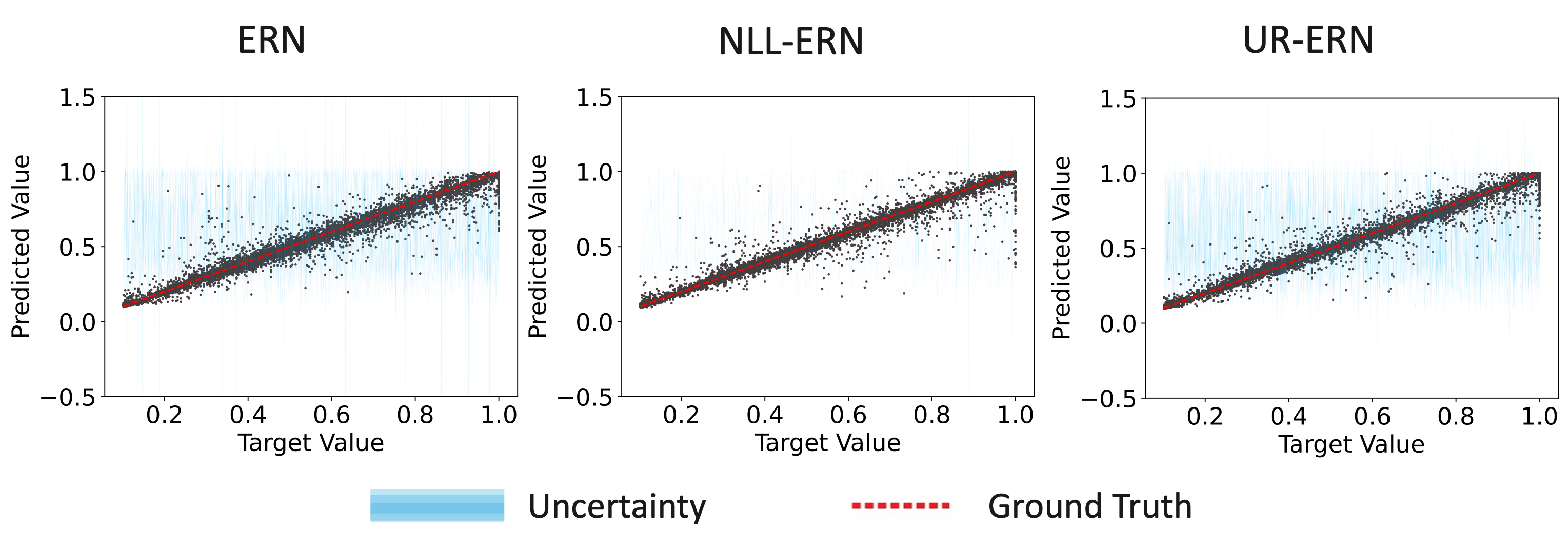} 
\caption{Performance of model outside HUA of Depth Estimation. The blue shade represents prediction uncertainty. A good estimation of uncertainty should cover the gap between prediction and ground truth exactly. \ours performs stably well compared with ERN and NLL-ERN both within HUA and outside HUA.}
\label{depth_outside_HUA_appen}
\end{figure}

% \begin{figure}
%     \centering
%         \subfloat[Ground Truth]{
%         \includegraphics[width=0.45\linewidth]{data.png}
%         %\label{fig:image_cutoff}
%     }
%     \subfloat[Ground Truth]{
%         \includegraphics[width=0.45\linewidth]{data_xy.png}
%         %\label{fig:image_cutoff}
%     } \\
%         \subfloat[Multivariate ERN]{
%         \includegraphics[width=0.45\linewidth]{model_xy_nll.png}
%         %\label{fig:image_calib}
%     }
%     \subfloat[\ours]{
%         \includegraphics[width=0.45\linewidth]{model_xy.png}
%         %\label{fig:image_calib}
%     }
%     \caption{Prediction of regression target. The value of $t$ is color coded.\hua{can you add explaination to the colors, what does it mean, and what's the conclusion of this figure?}}
%     \label{fig:multivariate_appen}
% \end{figure}

\subsubsection{Depth Estimation}For depth estimation, we also follow the setup of \citeauthor{NEURIPS2020_aab08546}. Depth estimation is conducted on  NYU-Depth-v2 dataset~\cite{silberman2012indoor}. Each image scan's missing depth holes were filled using the Levin Colorization method, and the resulting depth map was inverted to be proportional to disparity, as commonly practiced in depth learning literature. 
The resulting images were saved and used for training a U-Net \cite{ronneberger2015u} backbone model with five convolutional and pooling blocks. The dataset was randomly split into training, validation, and test sets (80-10-10). The input images had a shape of $(160,128)$ with 3 feature maps (RGB), while the target had a single disparity map. 
The training hyperparameters included a batch size of 32, Adam optimization with learning rate $5 \times 10^{-5}$, over 60000 iterations, and $\lambda=0.1$ for ERN. The best model by validation set RMSE was saved for testing. For experiments within HUA, we give a large negative bias to the activation layer to make the samples fall within HUA at the beginning of training.

For out-of-distribution experiments, We use images from ApolloScape \cite{huang2018apolloscape}, an OOD dataset for outdoor driving.

\begin{figure}
    \centering
        \subfloat[Ground Truth]{
        \includegraphics[width=0.45\linewidth]{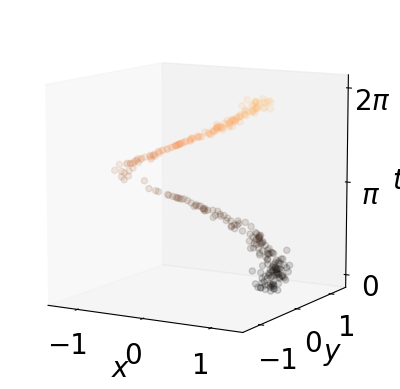}
        %\label{fig:image_cutoff}
    }
    \subfloat[Ground Truth ($xy$-plane)]{
        \includegraphics[width=0.45\linewidth]{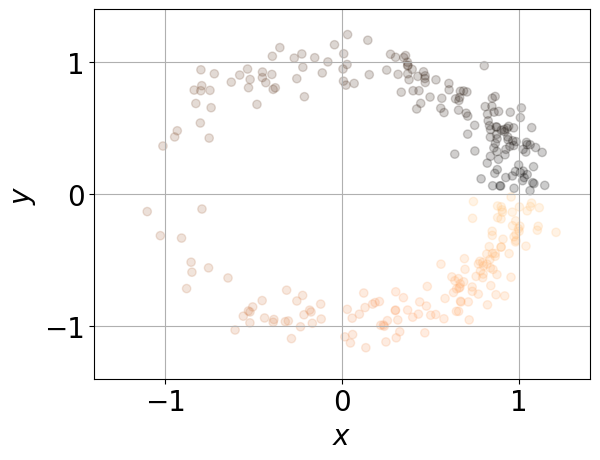}
        %\label{fig:image_cutoff}
    } \\
        \subfloat[Multivariate ERN]{
        \includegraphics[width=0.45\linewidth]{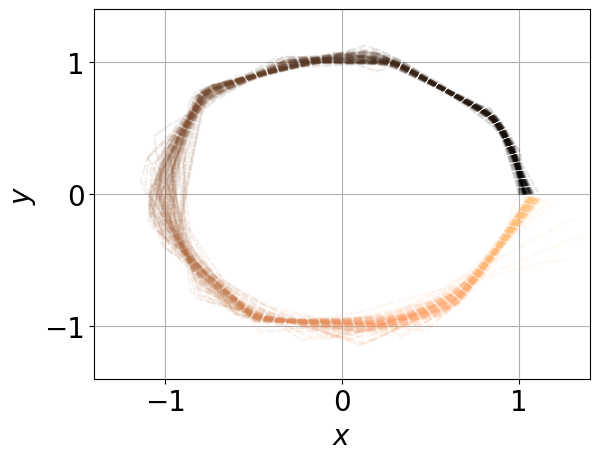}
        %\label{fig:image_calib}
    }
    \subfloat[\ours]{
        \includegraphics[width=0.45\linewidth]{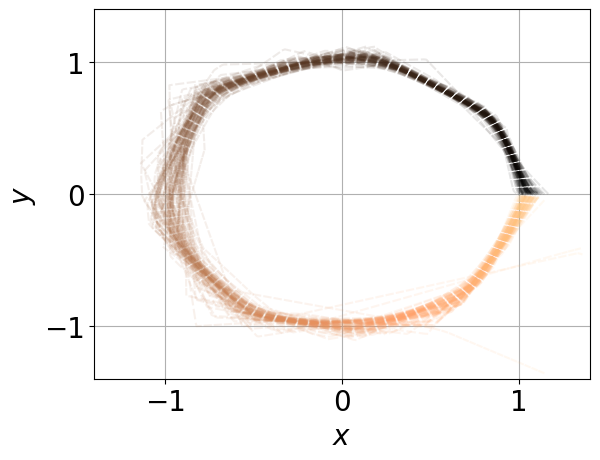}
        %\label{fig:image_calib}
    }
    \caption{Prediction of regression target. The value of $t$ is color coded. The larger value of $t$ corresponds to lighter color in the $xy$-plane figures. Besides estimating uncertainty, \ours performs stably well compared with Multivariate ERN in terms of predicting regression targets.}
    \label{fig:multivariate_appen}
\end{figure}

\subsubsection{Multivariate ERN}
For experiments of Multivariate ERN, we follow the setup of \citeauthor{meinert2021multivariate}. The target is to predict $(x, y) \in \mathbb{R}^2$ given $t \in \mathbb{R}$, where $x$ and $y$ being the features of the data sample given input $t$ with the following definition:
\begin{equation}
x=(1+\epsilon) \cos t, \quad 
y=(1+\epsilon) \sin t,
\end{equation}
and the distribution of $t$ is formulated as following:
\begin{equation}
t \sim \begin{cases}1-\frac{\zeta}{\pi} & \text { if } \zeta \in[0, \pi] \\ \frac{\zeta}{\pi}-1 & \text { if } \zeta \in(\pi, 2 \pi] \\ 0 & \text { else }\end{cases}
\end{equation}
where $\zeta \in[0,2 \pi]$ is uniformly distributed and $\epsilon \sim \mathcal{N}(0,0.1)$ is drawn from a normal distribution. We fit a distribution to 300 data points using a small neural network (NN) with one input neuron, two hidden layers of 32 neurons with Rectified Linear Unit activation, and six output neurons. The output $\vec{p} \in \mathbb{R}^6$ is transformed into the parameters of the evidential distribution.:
\begin{equation}
\begin{gathered}
\vec{\mu}=\left(\begin{array}{c}
p_1 \\
p_2
\end{array}\right), \quad \boldsymbol{L}=\left(\begin{array}{cc}
\exp \left\{p_3\right\} & 0 \\
p_4 & \exp \left\{p_5\right\}
\end{array}\right), \\
\nu=8+5 \tanh p_6,
\end{gathered}
\end{equation}
where an exponential function is used to constrain the diagonal elements of $\boldsymbol{L}$ to be strictly positive and the transformation of $\nu$ corresponds to the required lower bound $ \nu>n+1=3$.
All the experiments regarding Multivariate ERN are within HUA, we give a large negative bias to the activation layer to make the samples fall within HUA at the beginning of training.

\begin{figure}
\centering
\includegraphics[width=0.9\columnwidth]{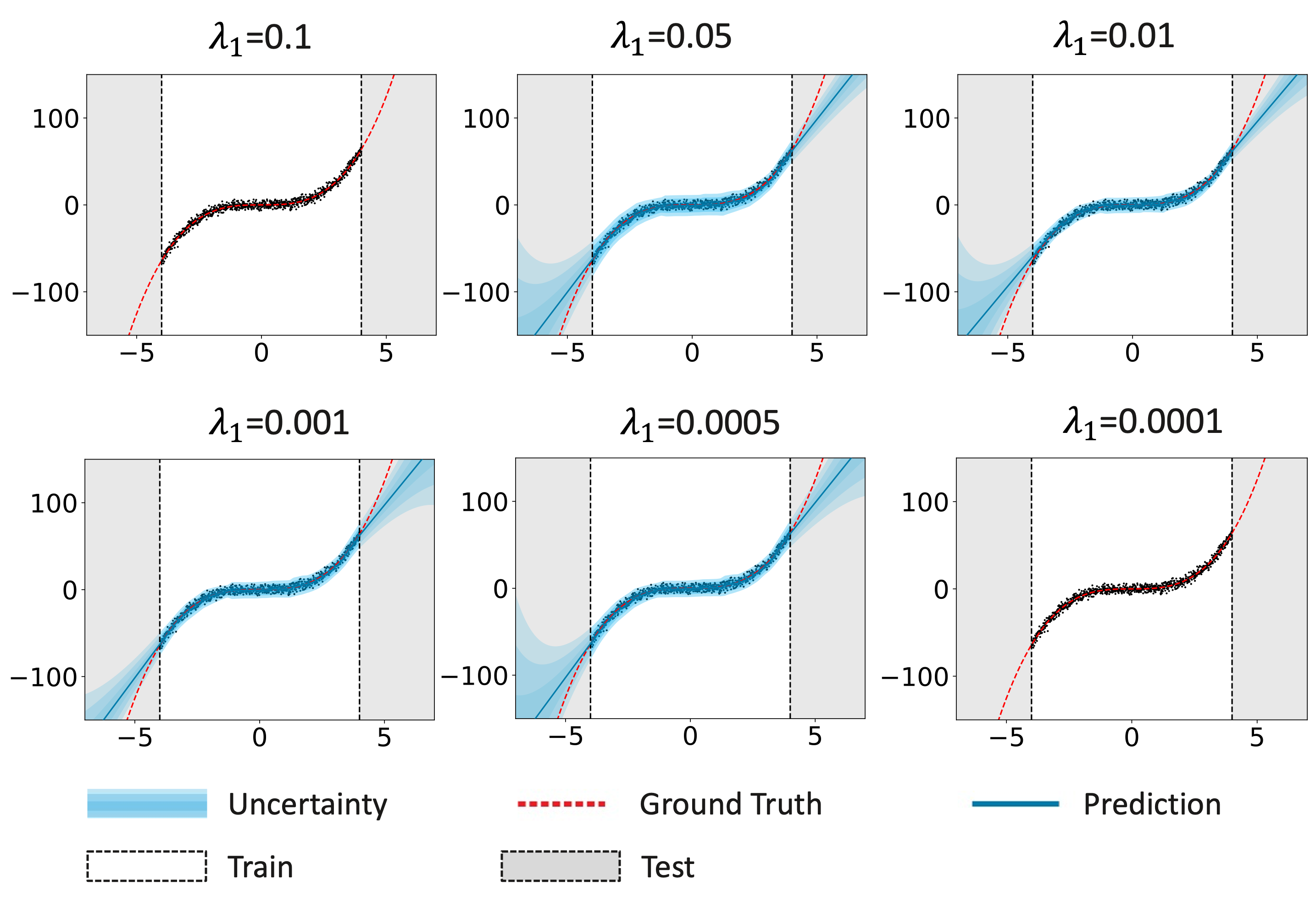} 
\caption{Impact of Uncertainty Regularization on model performance. The Figure shows model performance with different coefficients. 
Too small or too large coefficient $\lambda_{1}$ will make the regularization term infinite, failing to correctly estimate uncertainty.}
\label{sensitivity_outside_HUA_appen}
\end{figure}

\subsection{Additional Experiments}
\label{appendix_22}
\subsubsection{Performance on Real-World Benchmark}
We further evaluate the performance of models on the real-world benchmark, the UCI regression dataset \cite{hernandez2015probabilistic}. Our UCI regression experimental setup is the same as \citeauthor{hernandez2015probabilistic} (therefore outside HUA). Following previous works \cite{NEURIPS2020_aab08546,oh2022improving}, we use root mean squared error (RMSE) and negative log-likelihood (NLL) as evaluation metrics.

\begin{table}
\centering
\resizebox{\columnwidth}{!}{%
\begin{tabular}{@{}cccc@{}}
\toprule
\multirow{2}{*}{Datasets} & \multicolumn{3}{c}{RMSE}                                         \\ \cmidrule(l){2-4} 
                          & ERN                 & NLL-ERN             & UR-ERN               \\ \midrule
Boston                    & 3.06(0.16)          & 2.97(0.27)          & \textbf{2.91(0.28)}   \\
Concrete                  & 5.85(0.15)          & 5.76(0.16) & \textbf{5.74(0.24)}           \\
Energy                    & 2.06(0.10)          & 1.93(0.13)          & \textbf{1.88(0.09)}  \\
Kin8nm                    & 0.09(0.00)          & 0.06(0.00)          & \textbf{0.06(0.00)}  \\
Navel                     & 0.00(0.00)          & 0.00(0.00)          & \textbf{0.00(0.00)}  \\
Power                     & 4.23(0.09)          & \textbf{3.01(0.10)} & 3.02(0.08)           \\
Protein                   & 4.64(0.03)          & 3.71(0.15)          & \textbf{3.70(0.15)}  \\
Wine                      & 0.61(0.02)          & 0.57(0.02)          & \textbf{0.56(0.03)}  \\
Yacht                     & 1.57(0.56)          & \textbf{1.38(0.51)} & 1.57(0.54)           \\ \midrule
\multirow{2}{*}{Datasets} & \multicolumn{3}{c}{NLL}                                          \\ \cmidrule(l){2-4} 
                          & ERN                 & NLL-ERN             & UR-ERN               \\ \midrule
Boston                    & 2.35(0.06)          & 2.33(0.06)          & \textbf{2.33(0.06)}  \\
Concrete                  & \textbf{3.01(0.02)} & 3.05(0.03)          & 3.05(0.04)           \\
Energy                    & 1.39(0.06)          & 1.35(0.03)          & \textbf{1.35(0.03)}  \\
Kin8nm                    & -1.24(0.01)         & -1.36(0.03)         & \textbf{-1.37(0.03)} \\
Navel                     & -5.73(0.07)         & -6.24(0.08)         & \textbf{-6.27(0.06)} \\
Power                     & 2.81(0.07)          & \textbf{2.54(0.02)} & 2.55(0.02)           \\
Protein                   & 2.63(0.00)          & \textbf{2.42(0.03)}          & 2.44(0.05)  \\
Wine                      & 0.89(0.05)          & \textbf{0.88(0.04)} & 0.89(0.05)           \\
Yacht                     & 1.03(0.19)          & \textbf{0.98(0.17)} & 1.00(0.15)           \\ \bottomrule
\end{tabular}%
}
\caption{Experimental results on UCI regression benchmark. Following previous works \cite{NEURIPS2020_aab08546,oh2022improving}, we use RMSE and NLL as evaluation metrics.
The best scores are highlighted and we report standard errors in the parentheses. The proposed \ours performs better than ERN and stably well compared with NLL-ERN.
}
\label{tab:UCI}
\end{table}

Table~\ref{tab:UCI} presents the comparison of model performance on UCI regression. As is shown in the table, even outside HUA, \ours still performs better than ERN. NLL-ERN focuses on predicting the target, but it lags behind when estimating uncertainty. Still, \ours performs stably well compared with NLL-ERN on UCI regression benchmark.

We can observe from the UCI regression experiments that the proposed regularization not only helps the model make proper uncertainty estimation but also helps make better predictions of the target (model prediction is $\gamma$). This is not surprising because the gradient of loss function with respect to $\gamma$ is given by (no activation function is applied to $\gamma$):
\begin{equation}
\begin{aligned}
\frac{\partial \mathcal{L}^{\mathrm{NLL}}}{\partial o_\gamma} &=\frac{\partial \mathcal{L}^{\mathrm{NLL}}}{\partial \gamma}  \\
&=\frac{2 \nu\left(\gamma-y_i\right)(\alpha+0.5)}{2 \beta(\nu+1)+\nu\left(\gamma-y_i\right)^2} 
\end{aligned}
\end{equation}
where the parameters ($\gamma, v, \alpha, \beta$) are output of a neural network. If some samples fall in HUA, $\alpha$ will be hard to update, in turn influencing the update of $\gamma$.

\subsubsection{Depth Estimation outside HUA}
% \begin{figure}
% \centering
% \includegraphics[width=0.9\columnwidth]{depth_outside_HUA_appen.png} 
% \caption{Performance of model outside HUA of Depth Estimation. The blue shade represents prediction uncertainty. A good estimation of uncertainty should cover the gap between prediction and ground truth exactly.}
% \label{depth_outside_HUA_appen}
% \end{figure}
Figure~\ref{depth_outside_HUA_appen} shows the model performance outside HUA. As is shown from the figure, NLL-ERN doesn't perform well in uncertainty estimation, which aligns with the previous study \cite{NEURIPS2020_aab08546}. Compared to the performance  within HUA, ERN performs much better without the problem of zero gradients within HUA. But still, \ours is slightly better than ERN, showing the effectiveness of our method.

\subsubsection{Multivariate ERN}
% \begin{figure}
%     \centering
%     \subfloat[Ground Truth]{
%         \includegraphics[width=0.3\linewidth]{data_xy.png}
%         %\label{fig:image_cutoff}
%     }
%         \subfloat[Multivariate ERN]{
%         \includegraphics[width=0.3\linewidth]{model_xy_nll.png}
%         %\label{fig:image_calib}
%     }
%     \subfloat[\ours]{
%         \includegraphics[width=0.3\linewidth]{model_xy.png}
%         %\label{fig:image_calib}
%     }
%     \caption{Prediction of regression target. The value of $t$ is color coded.}
%     \label{fig:multivariate_appen}
% \end{figure}
Figure~\ref{fig:multivariate_appen} shows the prediction of the target for different models. As has been discussed before, we only study $\nu$ (decided by $p_{6}$) in this paper. Therefore, zero gradients of $\nu$ don't affect the prediction of the regression target ($\vec{\mu}$), which can explain the performance in Figure~\ref{fig:multivariate_appen}.

\subsection{Impact of Uncertainty Regularization}
\label{appendix_23}
We conduct experiments on Cubic Regression to investigate the influence of the strength ($\lambda_{1}$) of the proposed Uncertainty Regularization. The sensitivity experiments are conducted with \ours and outside HUA.
By exploring different values for the coefficient $\lambda_{1}$, we assess how the regularization term impacts the model's performance. This analysis helps us understand the sensitivity of our method to this hyperparameter and informs the optimal selection of $\lambda_{1}$ for enhanced uncertainty prediction. 
Figure~\ref{sensitivity_outside_HUA_appen} shows model performance with different coefficients $\lambda_{1}$. As is shown in the Figure, too small or too large coefficient $\lambda_{1}$ will make the regularization term infinite, failing to correctly estimate uncertainty. But if the coefficient is within a certain range, we can see the performance of \ours is pretty stable.

\end{document}